\documentclass[10pt,twocolumn,letterpaper]{article}

\usepackage{cvpr}
\usepackage{times}
\usepackage{epsfig}
\usepackage{graphicx}
\usepackage{amsmath}
\usepackage{amssymb}
\usepackage{algorithmic}
\usepackage{caption}

\usepackage{times}
\usepackage{helvet}
\usepackage{courier}
\usepackage{lipsum}

\usepackage{bm}
\usepackage{color}
\usepackage[linesnumbered,ruled,vlined]{algorithm2e}
\usepackage{subfig}

\usepackage{array}
\usepackage{multirow}
\usepackage{booktabs}
\usepackage{rotating}

\usepackage{dcolumn} 
\newcolumntype{d}{D{.}{.}{4}} 

\usepackage{amsthm}

\usepackage[utf8]{inputenc}
\usepackage[english]{babel}

 \cvprfinalcopy 

\usepackage[pagebackref=true,breaklinks=true,letterpaper=true,colorlinks,bookmarks=false]{hyperref}

\newtheorem{theorem}{Theorem}

\newtheorem{lemma}{Lemma}

\begin{document}
\title{GeoDA: a geometric framework for black-box adversarial attacks}

\author{Ali Rahmati$^*$, Seyed-Mohsen Moosavi-Dezfooli$^\dagger$, Pascal Frossard$^\ddagger$, and Huaiyu Dai$^*$\\ \\
$^*$Department of ECE, North Carolina State University \\ 
$^\dagger$Institue for Machine Learning,
ETH Zurich\\
$^\ddagger$Ecole Polytechnique Federale de
Lausanne \\ {\tt\small arahmat@ncsu.edu, seyed.moosavi@inf.ethz.ch, pascal.frossard@epfl.ch, hdai@ncsu.edu}}


\maketitle

\begin{abstract}
Adversarial examples are known as carefully perturbed images fooling image classifiers. 
We propose  a geometric framework to generate adversarial examples in one of the most challenging black-box settings where the adversary can only generate a small number of queries, each of them returning the top-$1$ label of the classifier. Our framework is based on the observation that the decision boundary of deep  networks usually has a small mean curvature in the vicinity of data samples. We propose an effective iterative algorithm to generate query-efficient black-box perturbations with small $\ell_p$ norms for $p \ge 1$, which is confirmed via experimental evaluations on state-of-the-art natural image classifiers. Moreover, for $p=2$, we theoretically show that our algorithm actually converges to the minimal $\ell_2$-perturbation when the curvature of the decision boundary is bounded. We also obtain the optimal distribution of the queries over the iterations of the algorithm. Finally, experimental results  confirm that our principled black-box attack algorithm performs better than state-of-the-art algorithms  as it generates smaller perturbations with a reduced number of queries.\footnote{The code of GeoDA is available at \url{https://github.com/thisisalirah/GeoDA}.} 
\end{abstract}

\vspace{-5mm}
\section{Introduction}

It has become well known that deep neural networks are vulnerable to small adversarial perturbations, which are carefully designed to cause miss-classification in state-of-the-art image classifiers~\cite{szegedy2013intriguing}.
Many methods have been proposed to evaluate adversarial robustness of classifiers in the white-box setting, where the adversary has full access to the target model~\cite{goodfellow2014explaining, moosavi2016deepfool, carlini2017towards}. However, the robustness of classifiers in black-box settings -- where the adversary has only access to the output of the classifier -- is of high relevance in many real-world applications of deep neural networks such as autonomous systems and healthcare, where it poses serious security threats. Several black-box evaluation methods have been proposed in the literature. Depending on what the classifier gives as an output, black-box evaluation methods are either score-based~\cite{narodytska2016simple, chen2017zoo, ilyas2018black} or decision-based~\cite{chen2019boundary, brendel2017decision, liu2019geometry}.

  \begin{figure}[!t]
	\centering
	
	\includegraphics[width=0.47\textwidth]{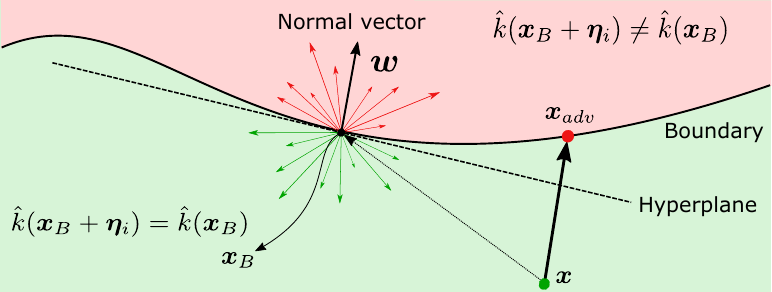}
	\label{fig:antenna_LMS1}\vspace{-1.5mm}
	\caption{Linearization of the decision boundary. }
	\label{fig2}
\end{figure}

In this paper, we propose a novel geometric framework for decision-based black-box attacks in which the adversary only has access to the \textit{top-$1$ label} of the target model. Intuitively small adversarial perturbations should be searched in directions where the classifier decision boundary comes close to data samples. We exploit the low mean curvature of the decision boundary in the vicinity of the data samples  to effectively estimate the normal vector to the decision boundary. This key prior permits to considerably reduces the number of queries that are necessary to fool the black-box classifier. Experimental results confirm that  our Geometric Decision-based Attack (GeoDA) outperforms state-of-the-art black-box attacks, in terms of required number of queries  to fool the classifier.
Our main contributions are summarized as follows:

\begin{itemize}

    \item We propose a novel  geometric framework based on linearizing the decision boundary of deep networks in the vicinity of samples. The error for the estimation of the  normal vector to the decision boundary of classifiers with flat decision boundaries, including linear classifiers, is shown to be bounded in a non-asymptotic regime. The proposed framework is general enough to be deployed for any classifier with low curvature decision boundary.
    
    \item We demonstrate how our proposed framework can be used to generate query-efficient $\ell_p$ black-box perturbations. In particular, we provide algorithms to generate perturbations for $p \ge 1$, and show their effectiveness via experimental evaluations on state-of-the-art natural image classifiers. In the case of $p=2$, we also prove that our algorithm  converges to the minimal $\ell_2$-perturbation. We further derive the optimal number of queries for each step of the iterative search strategy.
    
    
    \item Finally, we show that our framework can incorporate different prior information, particularly transferability and subspace constraints on the adversarial perturbations. We show theoretically that having prior information can bias the normal vector estimation search space towards a more accurate estimation. 
\end{itemize}

\section{Related work}

Adversarial examples can be crafted  in white-box setting~\cite{goodfellow2014explaining, moosavi2016deepfool, carlini2017towards}, score-based black-box setting~\cite{narodytska2016simple, chen2017zoo, ilyas2018black} or decision-based black-box scenario~\cite{chen2019boundary, brendel2017decision, liu2019geometry}. The latter settings are obviously the most challenging as little is known about the target classification settings.
Yet, there are several recent works on the black-box attacks on image classifiers  \cite{ilyas2018black,ilyas2018prior,tu2018autozoom}. However, they assume that the loss function, the prediction probabilities, or several top sorted labels are available, which may be unrealistic in many real-world scenarios. In the most challenging settings, there are a few attacks that exploit only the top-$1$ label information returned by the classifier, including the Boundary Attack (BA)~ \cite{brendel2017decision}, the HopSkipJump Attack (HSJA) \cite{chen2019hopskipjumpattack}, the OPT attack \cite{cheng2018query}, and qFool \cite{liu2019geometry}. In \cite{brendel2017decision}, by starting from a large adversarial perturbation, BA can iteratively reduce the norm of the perturbation.  In \cite{chen2019hopskipjumpattack}, the authors provided an attack based on \cite{brendel2017decision} that improves the BA taking the advantage of an estimated gradient. This attack is quite query efficient and can be assumed as the state-of-the-art baseline in the black-box setting. In \cite{cheng2018query},  an optimization-based hard-label black-box attack algorithm is introduced with guaranteed convergence rate in the hard-label black-box setting which outperforms the BA in terms of number of queries. Closer to our work, in \cite{liu2019geometry}, a heuristic algorithm based on the estimation of the normal vector to decision boundary is proposed for the case of $\ell_2$-norm perturbations. 

Most of the aforementioned attacks are however specifically designed for minimizing perturbation metrics such $\ell_2$ and $\ell_\infty$ norms, and mostly use heuristics. In contrast, we introduce a powerful and generic framework grounded on the geometric properties of the decision boundary of deep networks, and propose a principled approach to design efficient algorithms to generate general $\ell_p$-norm perturbations, in which \cite{liu2019geometry} can be seen as a special case. We also provide convergence guarantees for the $\ell_2$-norm perturbations. We  obtained the optimal distribution of queries over iterations theoretically as well which permits to use the queries in a more efficient manner.  Moreover, the parameters of our algorithm are further determined via empirical and theoretical analysis, not merely based on heuristics as done in~\cite{liu2019geometry}.
 
 
\vspace{-2mm}

 \section{Problem statement}
 

 Let us assume that we have a pre-trained $L$-class classifier with parameters $\theta$ represented as $f : \mathbb{R}^d \rightarrow \mathbb{R}^L $, where ${\bm x} \in \mathbb{R}^d$
is the input
image and  $\hat k ({\bm x}) = \textrm{argmax}_k ~ f_{ k}({\bm x})$ is the top-$1$ classification label where $f_k({\bm x} )$ is the $k^\textrm{th}$ component of $f({\bm x} )$ corresponds to the $k^\textrm{th}$ class.
We consider the non-targeted black-box attack, where
an adversary without any knowledge on $\theta$ computes an adversarial perturbation $\bm v$ to
change the estimated label of an image ${\bm x} $ to any incorrect label, i.e., $\hat k({\bm x} +\bm{v})\ne \hat k({\bm x} )$.  
 The distance metric $\mathcal D ({\bm x} , {\bm x}  + \bm v)$ can be any function including the $\ell_p$ norms.  We assume a general form optimization problem in which the goal is to fool the classifier while $\mathcal D ({\bm x} , {\bm x}  + \bm v)$ is minimized as:
 \vspace{-1mm}
 \begin{equation}\label{ffdd}
\begin{array}{rrclcl}
\displaystyle \min_{\bm v} & \multicolumn{3}{l}{\mathcal{D}({\bm x} , {\bm x}  + \bm v)}\\
\textrm{s.t.} & \hat k({\bm x} +\bm{v})\ne \hat k({\bm x} ).\\
\end{array}
\end{equation}
Finding a solution for ~\eqref{ffdd} is a hard problem in general. To obtain an efficient approximate solution, one can try to estimate the point of the classifier decision boundary that is the closest to the data point ${\bm x} $. Crafting an small adversarial perturbation then consists in pushing the data point beyond the decision boundary in the direction of its normal. The normal to the decision boundary is thus critical in a geometry-based attack. While it can be obtained using back-propagation in white box settings (e.g., \cite{moosavi2016deepfool}), its estimation in black-box settings becomes challenging.

The key idea here is to exploit the geometric properties of the decision boundary in deep networks for effective estimation in black-box settings. In particular, it has been shown that the decision boundaries of the state-of-the-art deep networks have a quite low mean curvature in the neighborhood of data samples~\cite{fawzi2016robustness}. Specifically, 
the decision boundary at the vicinity of a data point ${\bm x} $ can be locally approximated by a hyperplane passing through a boundary point $\bm{x}_B$ close to  ${\bm x} $, with a normal vector $\bm{w}$~\cite{fawzi2018empirical, fawzi2017robustness}. Thus, by exploiting this property, the optimization problem in~\eqref{ffdd} can be locally linearized as:
 \begin{equation}\label{ffdd1}
\begin{array}{rrclcl}
\displaystyle \min_{\bm v} & \multicolumn{3}{l}{\mathcal{D}({\bm x} , {\bm x}  + \bm v)}\\
\textrm{s.t.} & \bm w^T({\bm x}  + \bm v) - \bm w^T \bm x_B =0\\
\end{array}
\end{equation}
Typically, $\bm x_B$ is a point on the boundary, which can be found by binary search with a small number of queries. 
However, solving the problem  \eqref{ffdd1} is quite challenging  in black-box settings  as one does not have any knowledge about the parameters $\theta$ and can only access the top-1 label $\hat k({\bm x})$ of the image classifier. A \textit{query} is a request that results in the top-1 label of an image classifier for a given input, which prevents the use of zero-order black box optimization methods~\cite{zhao2019design, tu2019autozoom} that need more information to compute adversarial perturbations. The  goal of our method is to estimate the normal vector to the decision boundary $\bm w$ resorting to geometric priors with a minimal number of queries to the classifier.

\section{The estimator}
\label{sect:estimator}
We introduce an estimation method for the normal vector of classifiers with flat decision boundaries. It is worth noting that the proposed estimation is not limited to deep networks and applies to any  classifier with low mean curvature boundary. We denote the estimate of the  vector $\bm w$ normal to the flat decision boundary in~\eqref{ffdd1} with $\hat{\bm w}_N$ when $N$ queries are used. 
 Without loss of generality, we assume that the boundary point $\bm x_B$ is located at the origin. Thus, according to \eqref{ffdd1}, the decision boundary hyperplane passes through the origin and we have $\bm{w}^T \bm{x}=0$ for any vector $\bm{x}$ on the decision boundary hyperplane. In order to estimate the normal vector to the decision boundary, the key idea is to generate $N$ samples $\bm{\eta}_i,~ i\in \{1, \dots, N\}$ from a multivariate normal distribution $\bm{\eta}_i \sim \mathcal{N}(\bm{0},\bm\Sigma)$. Then, we query the image classifier $N$ times to obtain  the top-1 label output for each $\bm{x}_B+\bm{\eta}_i,~\forall i \in N$. For a given data point $\bm x$, if $\bm{w}^T \bm{x} \le 0$, the label is correct; if $\bm{w}^T \bm{x} \ge 0$,   the classifier is fooled. Hence, if the generated perturbations are adversarial, they belong to the set 
 \vspace{-1.5mm}
\begin{align}
        \mathcal{S}_{\textrm{adv}}& =  \{ \bm{\eta}_i \ | \ \hat k(\bm{x}_B+\bm{\eta}_i)\ne \hat k(\bm{x} ) \} \nonumber
    \\&= \{ \bm{\eta}_i \ | \ \bm{w}^T\bm{\eta}_i \ge 0 \}.
\end{align}
Similarly, the perturbations on the other side of the hyperplane, which lead to correct classification, belong to the set 
\begin{align}
        \mathcal{S}_{\textrm{clean}}& =  \{ \bm{\eta}_i \ | \ \hat k(\bm{x}_B+\bm{\eta}_i)= \hat k(\bm{x} ) \} \nonumber
    \\&= \{ \bm{\eta}_i \ | \  \bm{w}^T\bm{\eta}_i \le 0 \}.
\end{align}
The samples in each of the sets $\mathcal{S}_{\textrm{adv}}$ and $\mathcal{S}_{\textrm{clean}}$ can be assumed as samples drawn from a hyperplane ($\bm{w}^T \bm{x}=0$) truncated multivariate normal distribution with mean $\bm{0}$ and covariance matrix $\bm\Sigma$. We define the PDF  of the $d$ dimensional  zero mean multivariate normal distribution with covariance matrix $\bm\Sigma$ as $\phi_d(\bm{\eta}|\bm\Sigma)$. We define $ \Phi_d(\bm{b}|\bm\Sigma)=\int_{\bm{b}}^{\infty}\phi_d(\bm{\eta}|\bm\Sigma) d \bm{\eta}$    as  cumulative
distribution function of the univariate normal distribution.
\begin{lemma}
 Given a multivariate Gaussian distribution  $\mathcal{N}(\bm{0},\bm\Sigma)$ truncated by the hyperplane $\bm{w}^T\bm x \ge 0$, the mean $\bm \mu$ and covariance matrix $\bm R$ of the hyperplane truncated distribution are given by:
\begin{equation}\label{eq6}
    \bm{\mu}=c_1 \bm\Sigma \bm{w}
\end{equation}
where  $c_1=(\Phi_d(0))^{-1}\phi_d(0)$ and the covariance matrix $ \bm{R}=\bm\Sigma-\bm{\Sigma}\bm{w}\bm{w}^T\bm{\Sigma}(\Phi_d(0)^2 \gamma^2)^{-1} \phi_d(0))d^2(0)$
in which  $\gamma=(\bm w^T \bm\Sigma\bm w)^{\frac{1}{2}}$~\cite{tallis1965plane}.
\end{lemma}
As it can be seen in \eqref{eq6}, the mean is a function of both the covariance matrix $\bm\Sigma$ and $\bm w$. Our ultimate goal is to estimate the normal vector to the decision boundary. In order to recover $\bm w$ from $\bm \mu$, a sufficient condition is to choose $\bm \Sigma$ to be a full rank matrix.
\paragraph{General case} We first consider the case where no prior information on the search space is available. The matrix $\bm \Sigma=\sigma \mathcal{I}$ can be a simple choice to avoid  unnecessary computations. The direction of the mean of the truncated distribution is an estimation for the direction of hyperplane normal vector as $\bm{\mu}=c_1 \sigma \bm{w}$.
The covariance matrix of the truncated distribution  is $\bm{R}=\sigma \mathcal{I}+c_2\bm{w}\bm{w}^T$
where $c_2=-\sigma^2(\Phi_d(0))^{-2} \phi_d^2(0)$.
As the  samples in both of the sets $\mathcal{S}_{\textrm{adv}}$ and $\mathcal{S}_{\textrm{clean}}$ are hyperplane truncated Gaussian distributions, the 
 same estimation can be applied  for the samples in the set $\mathcal{S}_{\textrm{clean}}$ as well. Thus, by multiplying the samples in $\mathcal{S}_{\textrm{clean}}$ by $-1$ and we can use them to approximate the desired gradient to have a more efficient estimation.
Hence, the problem is reduced to the estimation of the mean of the $N$ samples drawn from the hyperplane truncated distribution with mean $\bm \mu$ and covariance matrix $\bm R$. As a result, the estimator $ \bar {\bm \mu}_N$  of $\bm \mu$ with $N$ samples is $ \bar {\bm \mu}_N=\frac{1}{N}{\sum_{i=1}^N\rho_i \bm \eta_i}$,
where
\begin{equation}
    \rho_i = \left\{\begin{matrix}
1 & \hat k(\bm{x}_B+\bm{\eta}_i)\ne \hat k(\bm{x} )\\ 
 -1& \hat k(\bm{x}_B+\bm{\eta}_i) = \hat k(\bm{x} ).
\end{matrix}\right.
\end{equation}
The normalized direction of the normal vector of the boundary can be obtained as:
\vspace{-2mm}
\begin{equation}\label{normal}
    \hat {\bm w}_N= \frac{\bar {\bm \mu}_N }{ \| \bar {\bm \mu}_N \|_2}
\end{equation}
\paragraph{Perturbation priors} We now consider the case where priors on the perturbations are available. In black-box settings, having prior information can significantly improve the performance of the attack. Although the attacker does not have access to the weights of the classifier, it may have  some prior information about the data, classifier, etc.  \cite{ilyas2018prior}. Here, using $\bm \Sigma$, we can capture the prior knowledge for the estimation of the normal vector to the decision boundary. In the following, we unify the two common priors in our proposed estimator.
In the first case, we have some prior information about the subspace in which we search for normal vectors, we can incorporate such information into $\bm \Sigma$ to have a more efficient estimation. For instance, deploying low frequency sub-space $\mathcal{R}^m$ in which $m\ll d$, we can generate a rank $m$ covariance matrix $\bm \Sigma$. 
Let us assume that $\mathcal S =\{\bm s_1, \bm s_2, ..., \bm s_m\}$ is an orthonormal  Discrete Cosine Transform (DCT) basis in the
$m$-dimensional subspace of the input space~\cite{guo2019simple}. In order to generate the samples from this low dimensional subspace, we use the following covariance matrix:
\begin{equation}
    \bm \Sigma=\frac{1}{m}\sum_{i=1}^{m} \bm s_i \bm s_i^T.
\end{equation}
The normal vector of the boundary can be obtained by plugging the modified $\bm \Sigma$ in \eqref{eq6}.
Second, we consider transferability priors. It has been observed that adversarial perturbations well transfer across different trained models~\cite{tramer2017space, moosavi2017universal, cheng2019improving}. Now, if the adversary further has full access to another model $\mathcal{T}'$, yet different than the target black-box model $\mathcal{T}$, it can take advantage of the transferability properties of adversarial perturbations. For a given datapoint, one can obtain the normal vector to the decision boundary in the vicinity of the datapoint for $\mathcal{T'}$, and bias the normal vector search space for the black-box classifier. Let us denote the transferred direction with unit-norm vector $\bm g$. By incorporating this vector into  $\bm \Sigma$, we can bias the search space as:
 \begin{equation}\label{eq222}
     \bm \Sigma={\beta}\mathcal I + (1-\beta) \bm g \bm g^T 
 \end{equation}
 where $\beta \in [0,1]$ adjusts the trade-off between exploitation and exploration. Depending on how confident we are about the utility of the transferred direction, we can adjust its  contribution  by tuning the value of $\beta$. Substituting \eqref{eq222} into \eqref{eq6}, after normalization to $c_1$, one can get  
 \begin{equation}\label{eq63}
    \bm{\mu}=  {\beta} \bm{w}+ (1-\beta) \bm g \bm g^T  \bm{w},
\end{equation}
where the first term is the estimated normal vector to the boundary and the second term is the projection of the estimated normal vector on the transferred direction $\bm g$.
Having incorporated the prior information into $\bm \Sigma$, one can generate perturbations $\bm{\eta}_i \sim \mathcal{N}(\bm{0},\bm\Sigma)$ with the modified $\bm \Sigma$ in an  effective search space, which leads to a more accurate estimation of normal to the decision boundary. 


\vspace{-2mm}
\paragraph{Estimator bound} Finally, we are interested in quantifying the number of samples that are necessary for estimating the normal vectors in our geometry inspired framework. Given a real i.i.d. sequence, using the central limit theorem, if the samples have a finite variance, an asymptotic bound can be provided for the estimate. However, this bound is not of our interest as it is only asymptotically correct.
We are interested in bounds of similar form with non-asymptotic inequalities as the number of queries is limited~\cite{lugosi2019sub, hanson1971bound}. 

\begin{lemma}
The mean estimation  $\bar {\bm \mu}_N$  deployed in \eqref{eq222} obtained from $N$ multivariate hyperplane truncated Gaussian queries satisfies the probability 
\begin{equation}\label{dava}
P \left(\|\bm{\bar \mu}_N - \bm{ \mu}\|\le \sqrt{\frac{\textrm{Tr}(\bm R)}{N}}+\sqrt{\frac{2\lambda_{\textrm{max}}\log(1/\delta)}{N}} \right) \ge 1-\delta
\end{equation}
where $\textrm{Tr}(\bm R)$ and $\lambda_{\textrm{max}}$ denote the trace and largest eigenvalue of the covariance
matrix $\bm R$, respectively. 
\end{lemma}
\begin{proof}
The proof can be found in Appendix A.
\end{proof}
This bound will be deployed in sub-section~\ref{l2p} to compute the optimal distribution of queries over iterations.

 \begin{algorithm}[t]
	
	\DontPrintSemicolon
	 \textbf{Inputs:} Original image $\bm x$, query budget $N$, $\lambda$, number of iterations $T$.
	 
	 \textbf{Output:} Adversarial example $\bm x_T$.

	 Obtain the optimal query distribution $N_t^*, \forall t$ by~\eqref{optq}.
	
	 Find a starting point on the boundary $\bm x_0 $.
	
	\For{$t=1:T$}{
		Estimate normal $\hat {\bm w}_{N^*_t}$ at $\bm x_{t-1}$ by $N_t^*$ queries.
		
			Obtain $\bm v_{t}$ according to \eqref{eq282}.
			
		$\hat r_t \gets \min\{r'>0\::\:\hat{k}(\bm x+r' \bm v_{t})\neq\hat{k}(\bm x)\}$

$	\bm x_{t} \gets \bm x + \hat r_{t} \hat {\bm w}_{N^*_t}$
	
	
	
	}

	\caption{{$\ell_p$ GeoDA (with optimal query distribution) for $p>1$}}
	\label{alg}
\end{algorithm}

 \section{Geometric decision-based attacks (GeoDA)}
 Based on the estimator provided in Section~\ref{sect:estimator}, one can design efficient black-box evaluation methods. In this paper, we focus on the minimal $\ell_p$-norm perturbations, i.e., $\mathcal{D}(\bm x, \bm x + \bm v)= \|\bm v\|_p$. We first describe the general algorithm for $\ell_p$ perturbations, and then provide algorithms to find black-box perturbations for $p=1,2,\infty$. Furthermore, for $p=2$, we prove the convergence of our method. The linearized optimization problem in~\eqref{ffdd1} can be re-written as
 \begin{align}\label{optimiz1}
 &    \hspace{8mm}  \min_{\bm v} 
 & &  \hspace{- 10mm} \|\bm v\|_p \nonumber \\
 & \hspace{9mm} {\text{ s.t.}}
& & \hspace{-10mm} \bm w^T(\bm x + \bm v) - \bm w^T \bm x_B =0. 
\end{align}
In the black-box setting, one needs to estimate $\bm x_B$ and $\bm w$ in order to solve this optimization problem. The boundary point $\bm x_B$ can be found using a similar approach as~\cite{liu2019geometry}. Having $\bm x_B$, one then use the process described in Section~\ref{sect:estimator} to compute the estimator of $\bm w$ -- i.e.,$\hat{\bm w}_{N_1}$ -- by making $N_1$ queries to the classifier.
In the case of $p=2$, the estimated direction $\hat{\bm w}_N$ is indeed the direction of the minimal perturbation. This process is depicted in Fig.~\ref{fig2}. 

If the curvature of the decision boundary is exactly zero, the solution of this problem gives the direction of the minimal $\ell_p$ perturbation. However, for deep neural networks, even if $N\rightarrow\infty$, the obtained direction is not completely aligned with the minimal perturbation as these networks still have a small yet non-zero curvature (see Fig.~\ref{fig2c}).  Nevertheless, to overcome this issue, the solution $\bm v^*$ of~\eqref{optimiz1} can be used to obtain a  boundary point $\bm x_1=\bm x+ \hat r_1\bm v^*$ to the original image $\bm x$ than $\bm x_0$, for an appropriate value of $\hat r_1>0$. For notation consistency, we define $\bm x_0 = \bm x_B$. Now, we can again solve~\eqref{optimiz1} for the new boundary point $\bm x_1$. Repeating this process results in an iterative algorithm to find the minimal $\ell_p$ perturbation, where each iteration corresponds to solving~\eqref{optimiz1} once.
Formally, for a given image $\bm x $, let $\bm x_t$ be the boundary point estimated in the  iteration $t-1$. Also, let $N_t$ be the number of queries used to estimate the normal to the decision boundary $\hat {\bm w}_{N_t}$ at the iteration $t$. Hence, the (normalized) solution to~\eqref{optimiz1} in the $t$-th iteration, $\bm v_t$, can be written in closed-form as:
\vspace{-2mm}
   \begin{equation}\label{eq282}
     \bm v_{t}=\frac{1}{\|\hat {\bm w}_{N_t}\|_{\frac{p}{p-1}}} \odot \textrm{sign}(\hat {\bm w}_{N_t}),
 \end{equation} 
 for $p\in[1,\infty)$, where $\odot$ is the point-wise product. For the particular case of $p = \infty$, the solution of \eqref{eq282}  is simply reduced to:
    \begin{equation}\label{eq2822}
     \bm v_{t}=\textrm{sign}(\hat {\bm w}_{N_t}).
 \end{equation}
The cases of the $p=1, 2$ are presented later. In all cases, $\bm x_t$ is then updated according to the following update rule:
 \begin{equation}\label{eeqqeeqq}
     \bm x_t=\bm x + \hat r_t \bm v_{t}
 \end{equation}
 where $\hat r_{t}$ can be found using an efficient line search along $\bm v_{t}$. The general algorithm is summarized in Alg.~\ref{alg}.


\subsection{$\ell_2$ perturbation}\label{l2p}
In the $\ell_2$ case,  the update rule of \eqref{eeqqeeqq}  is reduced to $ \bm x_{t}=\bm x  + \hat r_{t} \hat {\bm w}_{N_t}$
 where $\hat r_{t}$ is the $\ell_2$ distance of $\bm x $ to the decision boundary at iteration $t$. 
We propose convergence guarantees and optimal distribution of queries over the successive iterations for this case.
\vspace{-5mm}
 \paragraph{Convergence guarantees}
 We prove that GeoDA converges to the minimal $\ell_2$ perturbation given that the curvature of the decision boundary is bounded.
  We define the curvature of the decision boundary as $\kappa = \frac{1}{R}$, where $R$ is the radius of the largest open ball included in the region that intersects with the boundary $\mathcal{B}$~\cite{fawzi2016robustness}. In case $N\rightarrow\infty$, then $\hat r_t \to r_t$ where $r_t$ is assumed as exact distance required to push the image $\bm x$ towards the boundary at iteration $t$ with direction $\bm v_t$. The following Theorem holds:
  \begin{theorem}
Given a classifier with  decision boundary of bounded  curvature with $\kappa r~<~1$, the sequence $\{\hat r_t\}$   generated by Algorithm~\ref{alg} converges linearly to the  minimum $\ell_2$ distance $r$ since we have:
\vspace{-3mm}
\begin{equation}\label{lim}
    \lim_{t\to\infty} \frac{\hat r_{t+1}-r}{\hat r_t-r}= \lambda
\end{equation}
where $\lambda < 1$ is the convergence rate.
\end{theorem}
\begin{proof}
The proof can be found in Appendix B.
\end{proof}
\vspace{-6mm}
\paragraph{Optimal query distribution}
In practice, however, the number of queries $N$ is limited.
One natural question is  how should one choose the number of queries in each iteration of GeoDA. It can be seen in the experiments that allocating a smaller number of queries for the first iterations and then increasing it  in each iteration can improve the convergence rate of the GeoDA. At early iterations, noisy normal vector estimates are fine because the noise is smaller relative to the potential improvement, whereas in later iterations noise has a bigger impact. This makes the earlier iterations cheaper in terms of queries, potentially speeding up convergence~\cite{devarakonda2017adabatch}. 

We assume a practical setting in which we have a limited budget $N$ for the number of  queries as the target system may block if the number of queries increases beyond a certain threshold \cite{chen2019stateful}.   The goal is to obtain  the optimal distribution of the queries over the iterations. 
\vspace{-2mm}
\begin{theorem}
Given a limited query budget $N$, the  bounds for the  GeoDA $\ell_2$ perturbation error for total number of iterations  $T$ can be obtained as:
\begin{align}\label{boundt}
   \hspace{-2mm} \lambda^{T} (r_{0}-r) - e(\bm N) \le \hat r_{t}-r \le \lambda^{T} (r_{0}-r) +  e(\bm N)
\end{align}
where $   e(\bm N)= \gamma\sum_{i=1}^{T}\frac{\lambda^{T-i} r_{i}}{\sqrt{N_{i}}}$ is the error due to limited number of queries,  $\gamma = \sqrt{\textrm{Tr}{(\bm R})} + \sqrt{2 \lambda_{\max} \log(1/\delta)}$ and $N_t$ is the number of queries to estimate the normal vector to the boundary at point ${\bm{x}}_{t-1}$, and  $r_0 = \|\bm x - {\bm{x}}_0 \| $.
\end{theorem}
\begin{proof}
The proof can be found in Appendix C.
\end{proof}

	
	 
    

	
	
	
	




 As in \eqref{boundt}, the error in the convergence is due to two factors: (\textit{i}) curvature of the decision boundary (\textit{ii}) limited number of queries.  If the number of iterations increases, the effect of the curvature can  vanish. However,  the term $\gamma\frac{ r_i }{\sqrt{N_{t}}}$ is not small enough as the number of queries is finite. Having unlimited number of the queries, the error term due to queries can   vanish as well. However, given a limited number of  queries, what should be the distribution of the queries to alleviate such an error? 
  We define the following optimization problem:
  \vspace{-5mm}
\begin{align}\label{eqn:qopt}
\min_{{ N_1, \dots, N_T}}
&\qquad   \sum_{i=1}^{T}\frac{\lambda^{-i}r_i }{\sqrt{N_{i}}}\nonumber  \\
\text{s.t.}
&\qquad \sum_{i=1}^{T} N_i \le N 
\end{align}
where the objective is to minimize the error $e(\bm N)$  while the query budget  constraint is met over  all iterations.
\vspace{-2mm}
\begin{theorem}
 The optimal numbers of  queries for \eqref{eqn:qopt} in each iteration form  geometric sequence with the common ratio  $  \frac{N_{t+1}^*}{N_{t}^*} \approx \lambda^{-\frac{2}{3}}$,
 where $0 \le \lambda \le 1$. Moreover, we have
   \vspace{-2mm}
  \begin{equation}\label{optq}
     N_t^* \approx \frac{ \lambda^{-\frac{2}{3}t}}{\sum_{i=1}^T\lambda^{-\frac{2}{3}i}}
     N.
 \end{equation}
\end{theorem}
\begin{proof}
 The  proof can be found in Appendix~D.
\end{proof}
 
\vspace{-3mm}

\begin{figure*}[ht!]
\vspace{-5mm}
  \subfloat[\label{l2all}]{%
      \includegraphics[ width=5.4 cm,height=4cm]{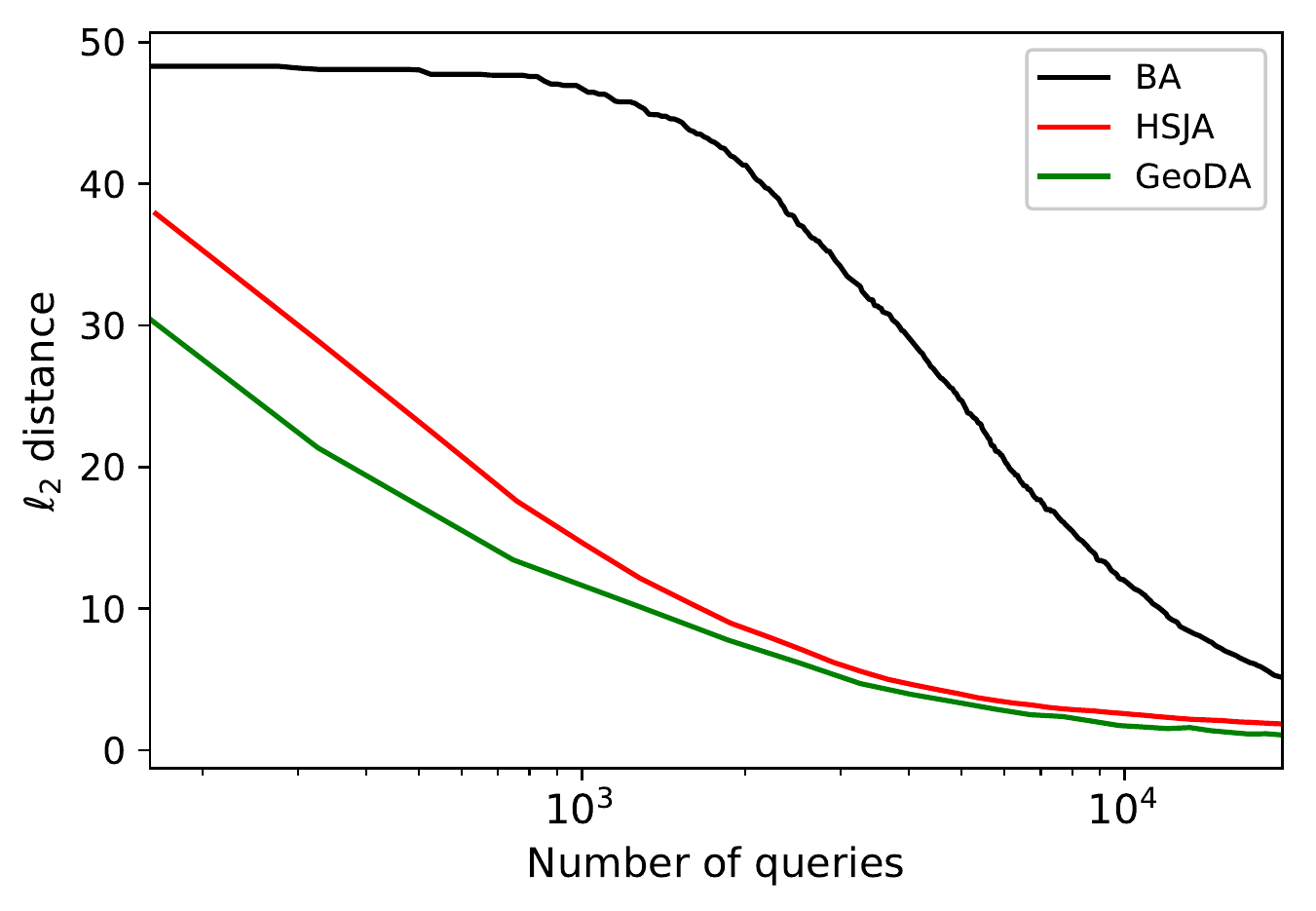}}
\hspace{\fill}
  \subfloat[\label{iterquery}   ]{%
      \includegraphics[ width=5.4 cm,height=4cm]{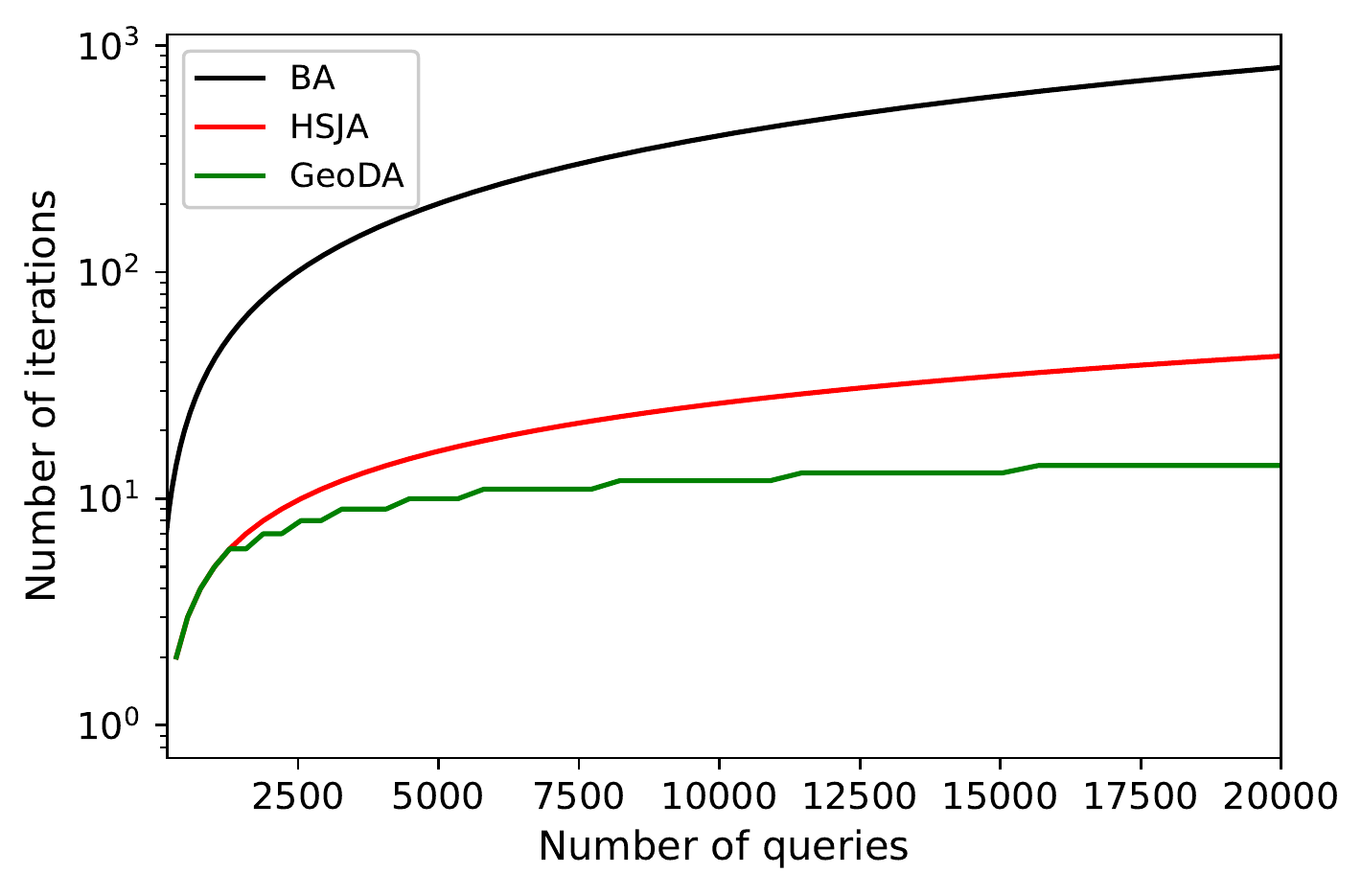}}
\hspace{\fill}
  \subfloat[\label{sparsesub}]{%
      \includegraphics[ width=5.4 cm,height=4cm]{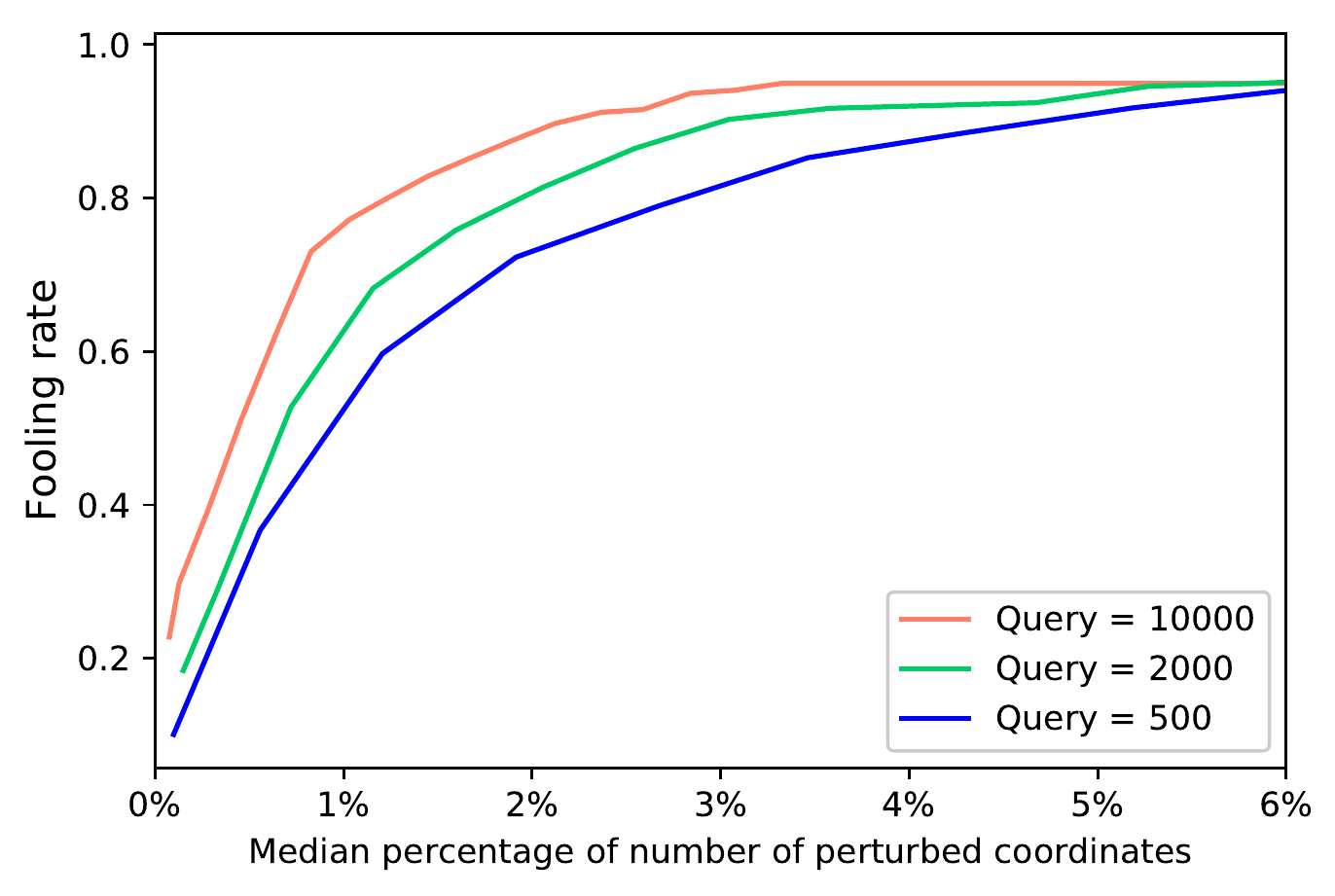}}\\
        \vspace{-3mm}
\caption{Performance evaluation of GeoDA for $\ell_p$ when $p=1, 2$ (a) Comparison for the performance of GeoDA, BA, and HSJA for $\ell_2$ norm. (b) Comparison for the number of required iterations in GeoDA, BA, and HSJA.  (c) Fooling rate vs. sparsity for different numbers of queries in sparse GeoDA. }
\end{figure*}

\subsection{$\ell_1$ perturbation (sparse case)}
The framework proposed by GeoDA is general enough to find sparse adversarial perturbations in the black-box setting as well. The sparse adversarial perturbations can be computed using the following optimization problem with box constraints as:
\begin{align}\label{optimiz}
 &    \hspace{8mm}  \min_{\bm v} 
 & &  \hspace{- 10mm} \|\bm v\|_1 \nonumber \\
 & \hspace{9mm} {\text{ s.t.}}
& & \hspace{-10mm} \bm w^T(\bm x + \bm v) - \bm w^T \bm x_B =0 \nonumber \\
& 
& & \hspace{-10mm} \bm l 	\preceq \bm x + \bm v 	\preceq \bm u
\end{align}
 In the box constraint $\bm l 	\preceq \bm x + \bm v 	\preceq \bm u$,   $\bm l$ and $\bm u$ denote the lower and upper bounds of the values of $\bm x + \bm v$.
 We can estimate  the normal  vector $\hat {\bm w}_N$ and the boundary point ${\bm x}_B$ similarly to the $\ell_2$ case with $N$ queries. Now, the decision boundary $\mathcal B $ is approximated with the hyperplane $  \{\bm x : \hat {\bm w}_N^T(\bm x-{\bm x}_B) = 0\}$. The goal is to find the top-$k$ coordinates of the normal vector $\hat {\bm w}_N$ with minimum $k$ and pushing them to extreme values of the valid range depending on the sign of the coordinate until it hits the approximated hyperplane. In order to find the minimum $k$, we deploy binary search for a $d$-dimensional image. Here, we just consider one iteration for the sparse attack., while the initial point of the sparse case is obtained using the GeoDA for $\ell_2$ case. The detailed Algorithm for the sparse version of GeoDA is given in Algorithm~\ref{alg2}.

   \begin{table*} 
\centering 
\begin{tabular}{c | c |c  c c c  } 
\toprule 

 & {{Queries}} & ~~  ~~$\ell_2$~~ ~~ &  ~~ ~~$\ell_\infty$~~ &~~~~ \normalsize{Iterations}~~ & ~~Gradients~~    \\

\midrule  \midrule 

& \small{1000}    & \small{47.92} &   \small{0.297}  & \small{40} & -\\
 Boundary attack~\cite{brendel2017decision}  & \small{5000}  & \small{24.67}  & \small{0.185} & \small{200}& - \\
    & \small{20000}  & \small{5.13}  & \small{0.052} &   \small{800} & - \\

\midrule

&  \small{1000}   & \small{16.05}  & - &   \small{3}& -  \\
 qFool~\cite{chen2019hopskipjumpattack}  & \small{5000}   & \small{7.52} &   - &  \small{3} & -\\
  & \small{20000}  & \small{1.12}  & - \  & \small{3}  & -\\
\midrule

&  \small{1000}   & \small{14.56}  & \small{0.062} &   \small{6}& -  \\
 HopSkipJump attack~\cite{chen2019hopskipjumpattack}  & \small{5000}   & \small{4.01} &   \small{0.031} &  \small{17} & -\\
  & \small{20000}  & \small{1.85}  & \small{0.012} \  & \small{42}  & -\\
\midrule

&  \small{1000}     & \small{{11.76}} &   \small{0.053}& \small{6}   & - \\
{GeoDA-fullspace}   & \small{5000}   & \small{{3.35}}  &   \small{0.022}  & \small{10}  & -\\
  & \small{20000}   & \small{{1.06}} &   \small{0.009}& \small{14}    & -\\
\midrule

&  \small{1000}   & \small{8.16}  & \small{0.022} & \small{6} & - \\
 {GeoDA-subspace}  & \small{5000}   & \small{2.51}  & \small{0.008} &    \small{10}  & -   \\
  & \small{20000}  & \small{1.01}  &     \small{0.003}    & \small{14} & -\\
  \midrule 

 DeepFool (white-box) \cite{moosavi2016deepfool} &  - & \small{0.026}  &  -  &  \small{2}   &  \small{20}   \\
   \midrule
  C\&W (white-box) \cite{carlini2017towards} &  - & \small{0.034}  & -  &  \small{10000}  &  \small{10000}   \\
\bottomrule
\end{tabular}

\caption{The performance comparison of GeoDA with BA and HSJA for median $\ell_2$ and $\ell_\infty$ on ImageNet dataset.} 
\label{l2linf} 
\end{table*}

\section{Experiments}
  \subsection{Settings}
 We evaluate our algorithms on a pre-trained ResNet-50~\cite{he2016deep} with a set $\mathcal X$ of 350 correctly classified and randomly selected images from the ILSVRC2012's validation set~\cite{deng2009imagenet}. All the images are resized to $224 \times 224 \times 3$.

 To evaluate the performance of the attack we deploy the median of the $\ell_p$ norm for $p = 2, \infty$ distance over all tested samples, defined by
$  \underset{{{{\bm x} \in \mathcal{X}} }}{\textrm{median}}\Big( \| \bm x - \bm x^{\textrm{adv}}\|_p \Big)$. For sparse perturbations, we measure the performance by fooling rate defined as $|\bm x \in \mathcal{X}: \hat k(\bm x) \ne \hat k(\bm x^{\textrm{adv}}) |/|\mathcal{X}|$. In evaluation of the sparse GeoDA, we define \textit{sparsity} as the percentage of the perturbed coordinates of the given image

 \begin{algorithm}[t]
	
	\DontPrintSemicolon
	 \textbf{Inputs:} Original image $\bm x$, query budget $N$, $\lambda$, projection operator $Q$.
	 
	 \textbf{Output:} Sparsely perturbed $\bm x_{\textrm{adv}}$, sparsity $s$.
    
    Obtain ${\bm x}_B$, $\hat {\bm w}_N$ with $N$ queries by $\ell_2$ GeoDA algorithm.

	$m_l = 0$, $m_u = d$, $J =\textrm{round}(\log_2(d))+1$
	
	\For{$j = 1:J$}{
	
	$k$ = round($\frac{m_u+m_l}{2}$)
	
	Obtain  top $k$  absolute values of coordinates of $\hat {\bm w}_{N}$ as $\hat {\bm w}_{\textrm{sp}}$.
	
	$\bm x_{j} \gets Q(\bm x  + {\nu}~ \textrm{sign}(\hat {\bm w}_{\textrm{sp}}))$

    \eIf{$\hat {\bm w}_N^T({\bm x}_j -{\bm x}_B) > 0$}{$m_u = k$}{$m_l = k$}

	} 
	$\bm x_{\textrm{adv}} \gets  \bm x_j$, $s \gets  m_u$

	\caption{{Sparse GeoDA }}
	\label{alg2}
\end{algorithm}

\begin{table} [h]
\centering 
\begin{tabular}{c | c |c c c c  } 
\toprule 

 & \small{\textbf{Queries}} & \small{\textbf{ Fooling rate}} & \small{\textbf{Perturbation} }    \\

\midrule  \midrule

  & \small{500 } & \small{88.44 $\%$}  & \small{4.29 $\%$}  \\
{{ GeoDA}} & \small{2000 }  &   \small{90.25 $\%$}& \small{3.04 $\%$}   \\
  & \small{10000 } & \small{91.17 $\%$}  & \small{2.36 $\%$}  \\

  \midrule 
SparseFool~\cite{brendel2017decision} & -  &   \small{100 $\%$} & \small{0.23 $\%$}  \\

\bottomrule
\end{tabular}

\caption{The performance comparison of black-box sparse GeoDA for median sparsity compared to white box attack SparseFool~\cite{brendel2017decision}  on ImageNet dataset.} 
\label{sparse} 
\end{table}



\subsection{Performance analysis}

\paragraph{Black-box attacks for $\ell_p$ norms.} We compare the performance of the GeoDA with state of the art attacks for $\ell_p$ norms. There are several attacks in the literature including  Boundary attack \cite{brendel2017decision}, HopSkipJump attack \cite{chen2019hopskipjumpattack}, qFool \cite{liu2019geometry}, and OPT attack \cite{cheng2018query}. In our experiments, we compare GeoDA with Boundary attack, qFool and HopSkipJump attack. We do not compare our algorithm with OPT attack as HopSkipJump already outperforms it considerably~\cite{chen2019hopskipjumpattack}.  In our algorithm, the optimal distribution of the queries is obtained for any given number of queries for $\ell_2$ case.   
The results for $\ell_2$ and $\ell_\infty$ for different numbers of queries is depicted in Table~\ref{l2linf}.  GeoDA can outperform the-state-of-the-art both in terms of smaller perturbations and number of iterations, which has the benefit of parallelization. In particular, the images can be fed into multiple GPUs with larger batch size.
In Fig.~\ref{l2all}, the $\ell_2$ norm of GeoDA, Boundary attack and HopSkipJump are compared. As shown, GeoDA can outperform the HopSkipJump attack especially when the number of queries is small. By increasing the number of queries, the performance of GeoDA and HopSkipJump are getting closer.  
In Fig.~\ref{iterquery}, the number of iterations versus the number of queries for different algorithms are compared. As depicted, GeoDA needs fewer iterations compared to HopSkipJump and BA when the number of queries increases. Thus, on the one hand GeoDA generates smaller $\ell_2$  perturbations compared to the HopSkipJump attack when the  number of queries is small, on the other hand, it saves significant computation time due to parallelization of queries fed into the GPU.

 \begin{figure*}
\centering
  \includegraphics[width=14.2cm,height=5cm]{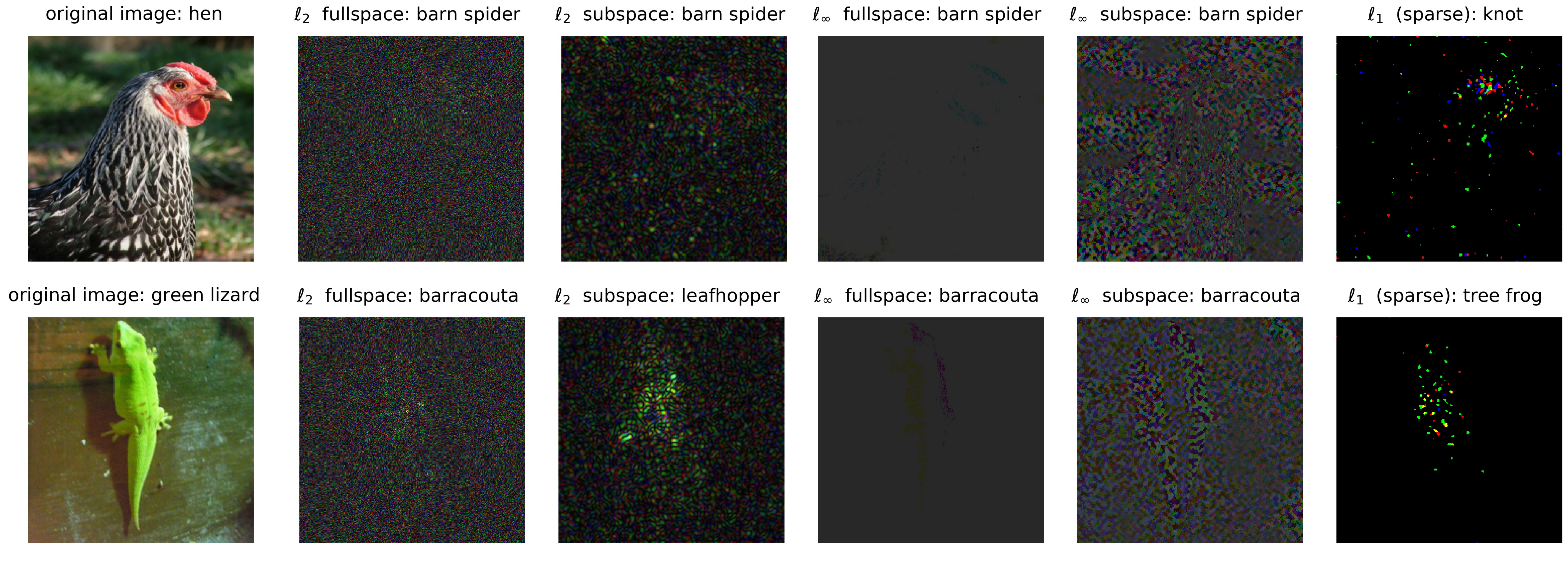}
  \label{pertss}
  \vspace{-3mm}
    \caption{Original images and adversarial perturbations generated by GeoDA for $\ell_2$ fullspace, $\ell_2$ subspace, $\ell_\infty$ fullspace, $\ell_\infty$ subspace, and $\ell_1$ sparse with $N=10000$ queries. (Perturbations are magnified $\sim10\times$ for better visibility.)}
    \label{pertss}
\end{figure*}

\begin{figure*}[ht]
  \subfloat[\label{fig2a}]{%
      \includegraphics[ width=5.2 cm,height=3.9cm]{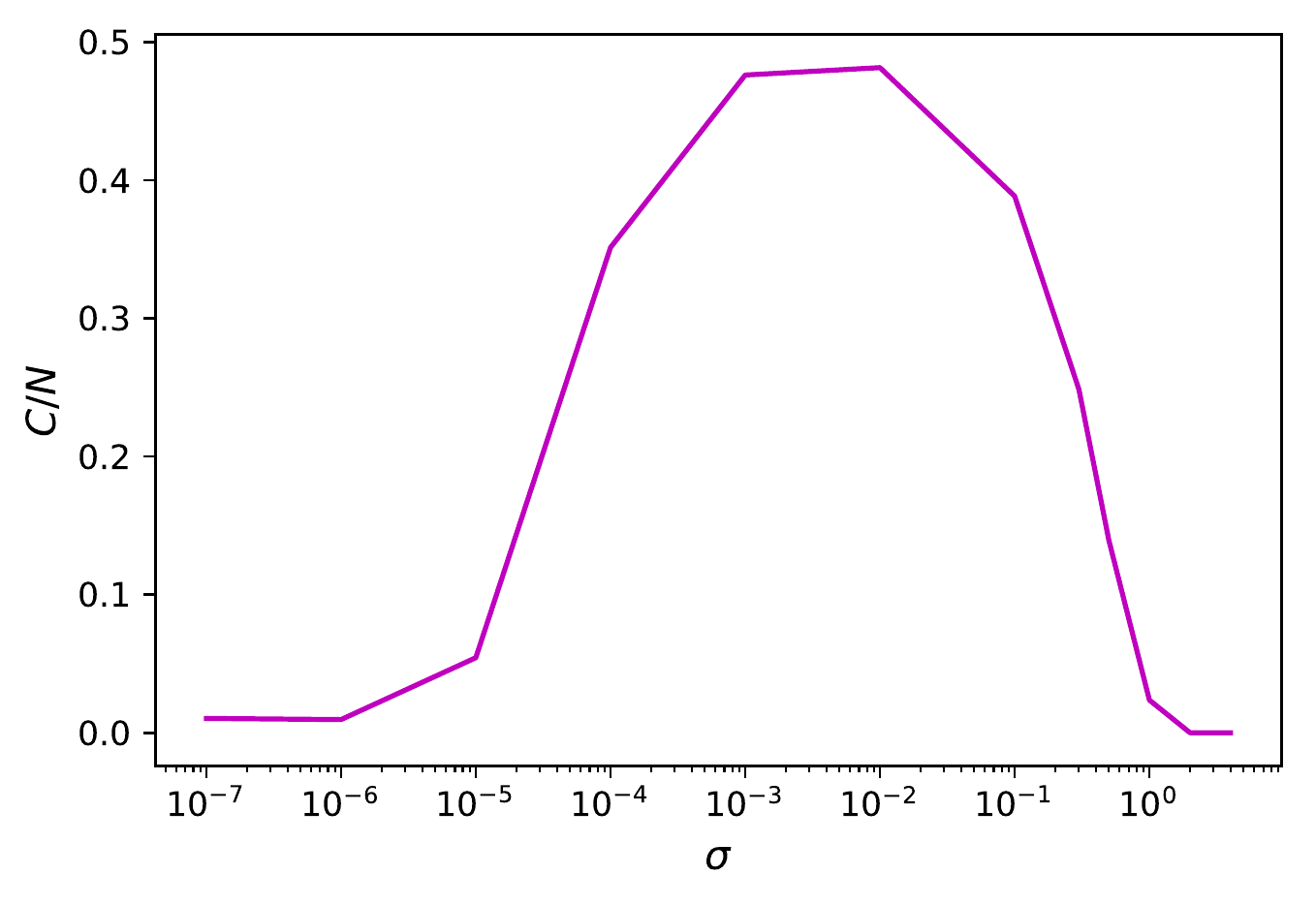}}
\hspace{\fill}
  \subfloat[\label{fig2b}   ]{%
      \includegraphics[ width=5.2 cm,height=3.9cm]{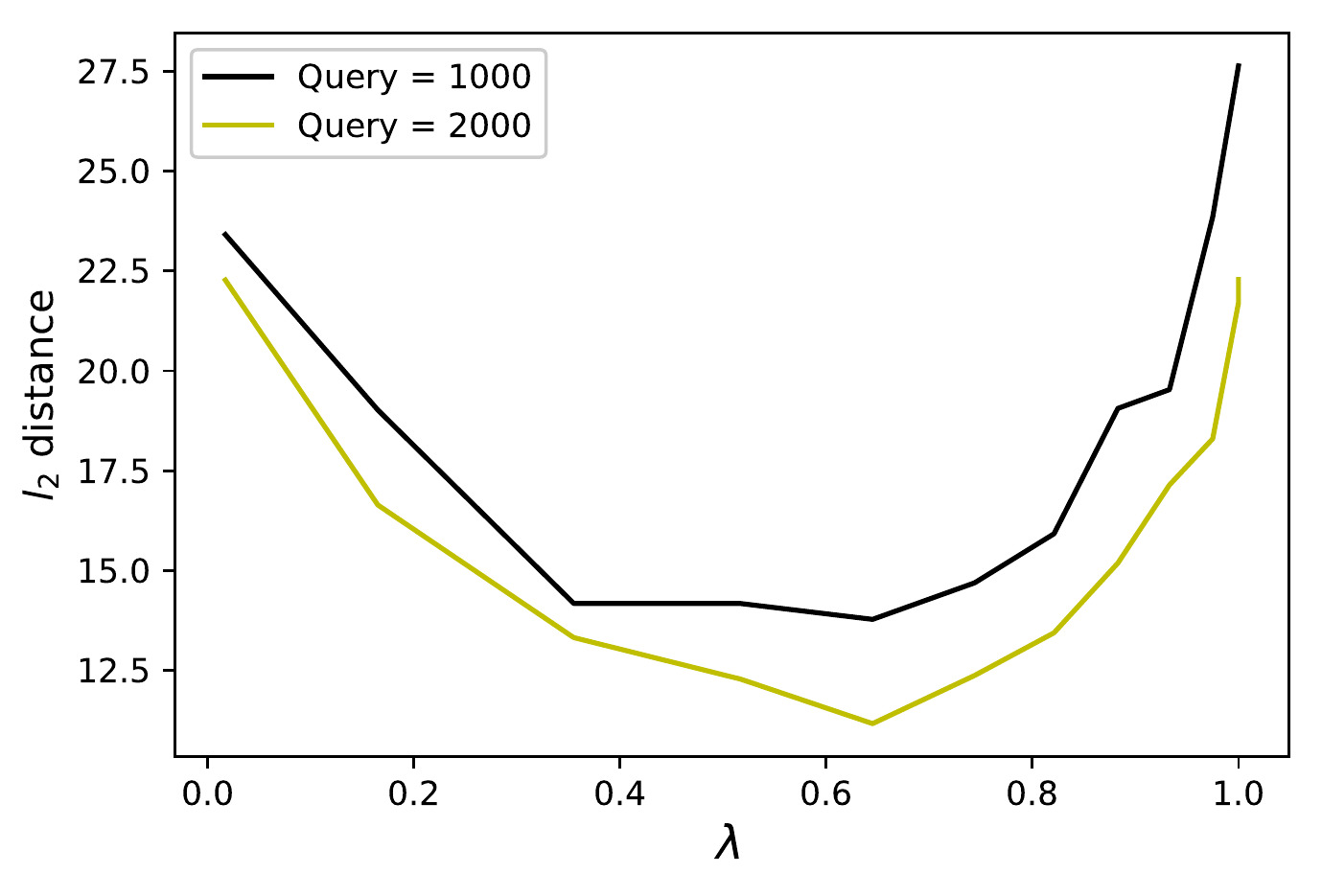}}
\hspace{\fill}
  \subfloat[\label{fig2c}]{%
      \includegraphics[ width=5.2 cm,height=3.9cm]{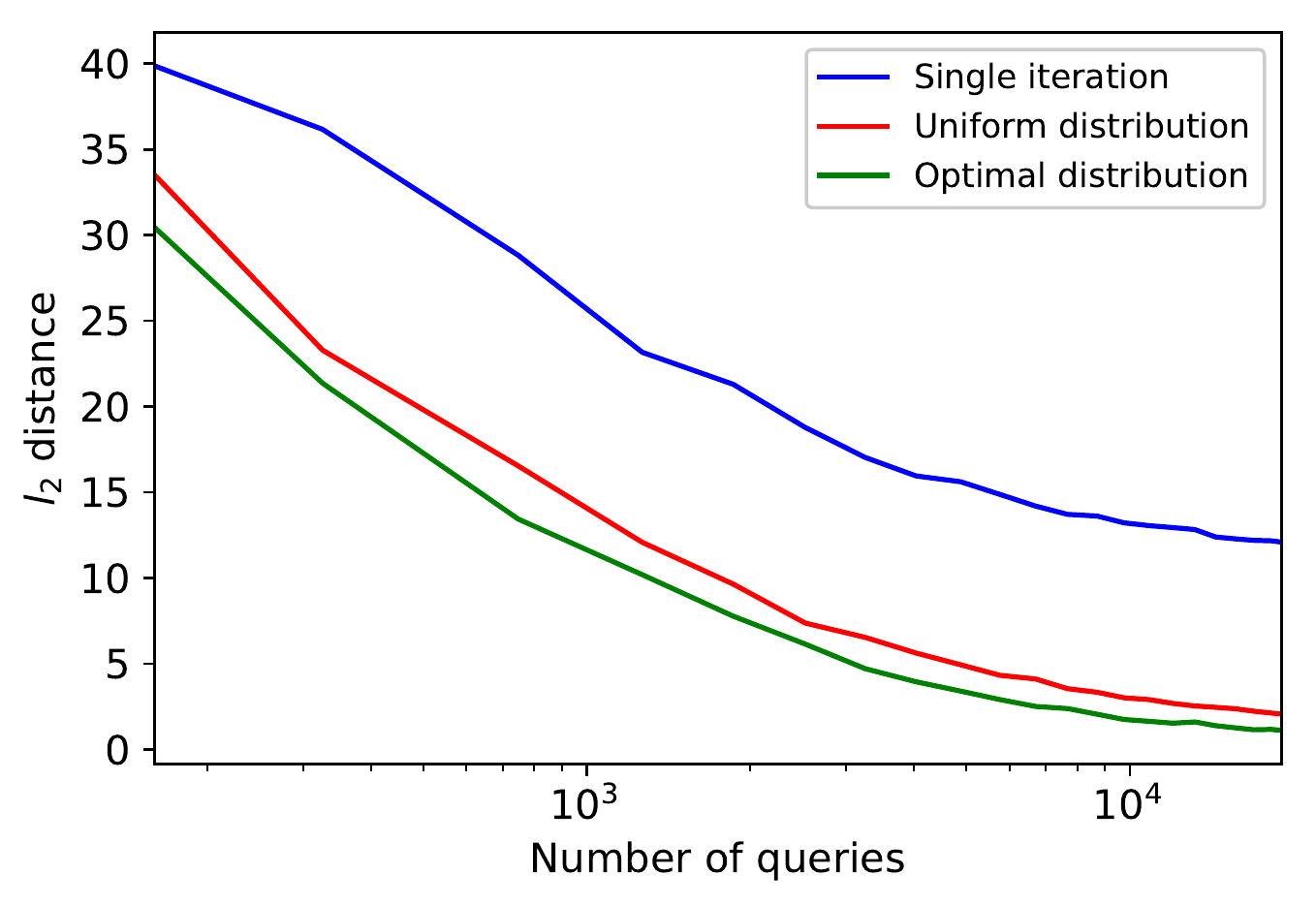}}\\
        \vspace{-3mm}
\caption{(a) The effect of the variance $\sigma$ on the ratio of correctly classified queries $C$ to the total number of queries $N$ at boundary point ${\bm x}_B$. (b) Effect of $\lambda$ on the performance of the algorithm.  (c) Comparison of two extreme cases of query distributions, i.e., single iteration ($\lambda \to 0$) and uniform distribution ($\lambda =1$) with optimal distribution ($\lambda = 0.6$). }
\end{figure*}


Now, we evaluate the performance of  GeoDA for generating sparse perturbations. In Fig.~\ref{sparsesub},  the fooling rate versus sparsity  is depicted. In experiments, we observed that instead of using the boundary point ${\bm x}_B$ in the sparse GeoDA, the performance of the algorithm can be improved by further moving towards the other side of the hyperplane boundary. Thus, we use ${\bm x}_B + \zeta({\bm x}_B - {\bm x} )$, where $\zeta \ge 0$. The parameter $\zeta$ can adjust the trade-off between the fooling rate and the sparsity. It is observed that the higher the value for $\zeta$,  the higher the fooling rate and the sparsity and vice versa. In other words, choosing small values for $\zeta$ produces sparser adversarial examples; however, it decreases the chance that it is an adversarial example for the actual boundary. In Fig.~\ref{sparsesub}, we depicted the trade-off between fooling rate and sparsity by increasing the value for $\zeta$ for different query budgets. 
The larger the number of queries, the closer the initial point to  the original image, and also the better our algorithm performs in generating sparse adversarial examples. In Table~\ref{sparse}, the sparse GeoDA is compared with the white-box attack SparseFool. We show that with a limited number of queries, GeoDA can generate sparse perturbations with acceptable fooling rate with sparsity of about $3 $ percent with respect to the white-box attack SparseFool.
The adversarial perturbations generated by GeoDA for  $\ell_p$ norms  are shown in Fig.~\ref{pertss} and the effect of different norms can be observed.

\vspace{-3mm}
\paragraph{Incorporating prior information.} Here, we evaluate the methods proposed in Section~\ref{sect:estimator} to incorporate prior information in order to improve the estimation of the normal vector to the decision boundary. As  sub-space priors, we deploy the DCT basis functions in which  $m$ low frequency subspace directions are chosen \cite{moosavi2019geometry22}. As shown in Fig. \ref{subtrans}, biasing the search space to the DCT sub-space can reduce the $\ell_2$ norm of the perturbations by approximately $ 27 \%$  compared to the full-space case. 
For transferrability, we obtain the normal vector of the given image using the white box attack DeepFool~\cite{moosavi2016deepfool} on a ResNet-34 classifier. We bias the search space for normal vector estimation as described in Section~\ref{sect:estimator}. As it can be seen in Fig.~\ref{subtrans}, prior information can improve the normal vector estimation significantly. 

 \subsection{Effect of hyper-parameters on the performance}
 Instead of throwing out the gradient obtained from the previous iterations, we can take advantage of them in next iterations as well. To do this, we can bias the covariance matrix $\bm \Sigma$ towards the gradient obtained from the previous iteration. The other way is to simply have a weighted average of the estimated gradient and previous gradients. As a general rule, $\beta$ given in \eqref{eq63} should be chosen in such a way that the estimated gradient in recent iterations  get more weights compared to the first iterations.
 
    \begin{figure}[ht]
	\centering
	\hspace*{-0.0in}
	\includegraphics[width=0.32\textwidth]{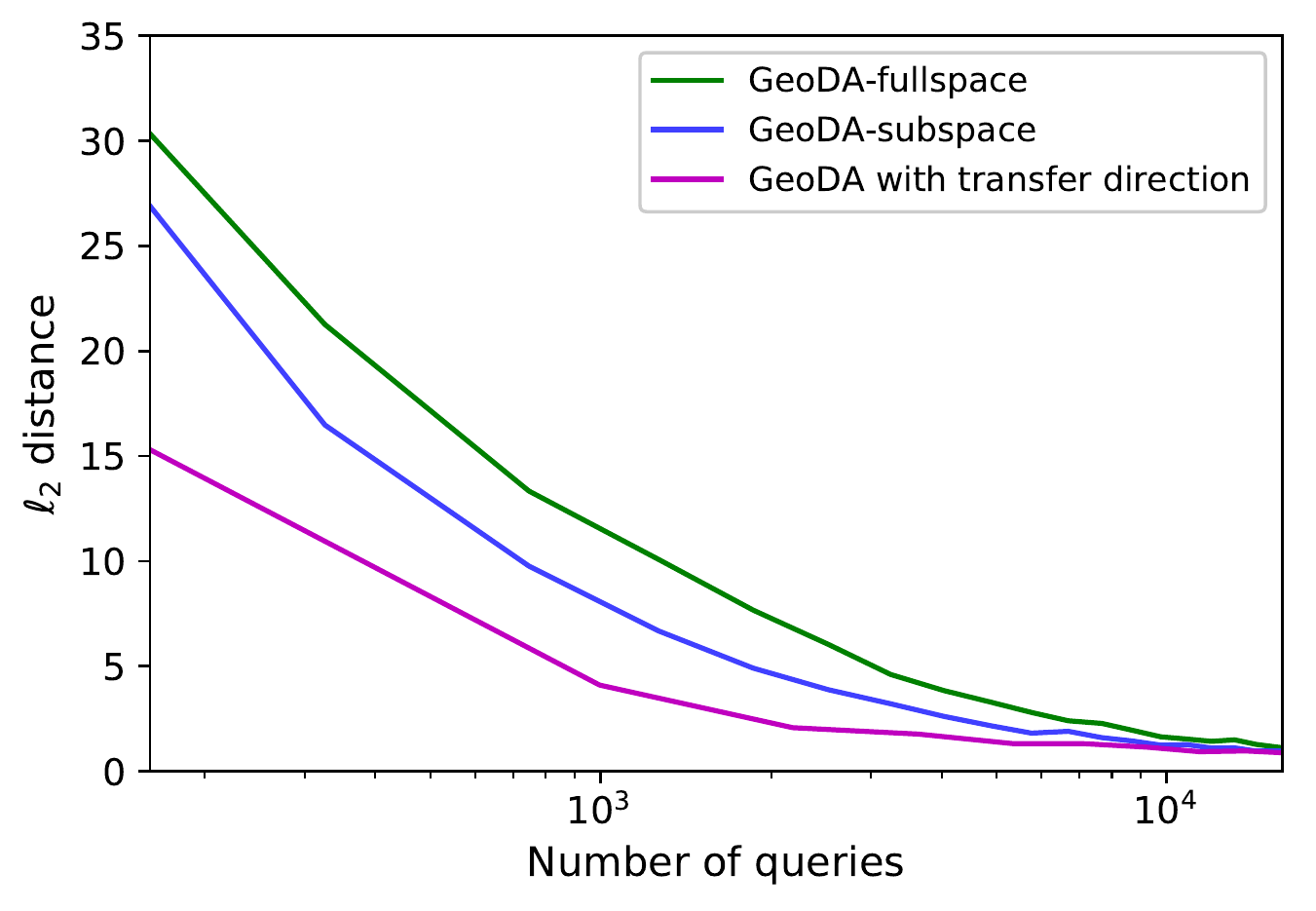}
	\label{fig:11}
	  \vspace{-3mm}
	\caption{Effect of prior information, i.e., DCT sub-space and transferability on the performance of $\ell_2$ perturbation.}
	\label{subtrans}
\end{figure}
 In practice, we need to choose $\sigma$ such that the locally flat assumption of the boundary is preserved. Upon generating the  queries at boundary point $\bm x_B$ to estimate  the direction of the normal vector as in~\eqref{normal}, the value for $\sigma$ is chosen in such a way that the number of correctly classified images and adversarial images on the boundary are almost the same. In Fig.~\ref{fig2a}, the effect of variance $\sigma$ of added Gaussian perturbation on the number of correctly classified queries on the boundary point is illustrated. We obtained a random point $\bm x_B $ on the decision boundary of the image classifier and query the image classifier 1000 times.  As it can be seen, the variance $\sigma$ is too small, none of the queries is correctly classified as the point $\bm x_B$ is not exactly on the boundary. It is worth mentioning that in binary search we choose the point on the adversarial side as a boundary point. On the other hand if the variance is too high,  all the images are classified as adversarial since they are highly perturbed.

In order to obtain the optimal query distribution for a given limited budget $N$, the values for $\lambda$ and $T$ should be given. Having fixed $\lambda$, if $T$ is large, the number of queries allocated to the first iteration may be too small. To address this, we consider a fixed number of queries for the first iteration as  $N^*_1 = 70$. Thus, having fixed  $\lambda$, a reasonable choice for $T$ can be obtained by solving \eqref{optq} for $T$. Based on \eqref{optq}, if $\lambda \to 0$, all the queries are allocated to the last iteration and when $\lambda = 1$, the query distribution is uniform. A value between these two extremes is desirable for our algorithm. To obtain this value, we run our algorithm for different $\lambda$ for only 10 images different from $\mathcal{X}$. As it can be seen in Fig. \ref{fig2b}, the algorithm has its worst performance when $\lambda$ is close to the two extreme cases: single iteration ($\lambda \to 0$) and uniform distribution ($\lambda = 1$). We thus choose the value $\lambda=0.6$ for our experiments. Finally, in Fig.~\ref{fig2c}, the comparison between three different query distributions is shown. The  optimal query distribution achieves the best performance while the single iteration preforms worst. Actually, the fact that the single iteration performs worst is reflected in our proposed bound in \eqref{boundt} as even with infinite number of queries it can not do better than $\lambda(r_0 - r)$. Indeed the effect of curvature can be addressed only by increasing the number of iterations.

\section{Conclusion}
In this work, we propose a new geometric framework for designing query-efficient decision-based black-box attacks, in which the attacker only has access to the top-1 label of the classifier. Our method relies on the key observation that the curvature of the decision boundary of deep networks is small in the vicinity of data samples. This permits to estimate the normals to the decision boundary with a small number of queries to the classifier, hence to eventually design query-efficient $\ell_p$-norm attacks. In the particular case of $\ell_2$-norm attacks, we show theoretically that our algorithm converges to the minimal adversarial perturbations, and that the number of queries at each step of the iterative search can be optimized mathematically. We finally study  GeoDA through extensive experiments that confirm its superior performance compared to state-of-the-art black-box attacks. 

\subsection*{Acknowledgements}
This work was supported in part by the US National Science Foundation under grants ECCS-1444009 and CNS-1824518. S. M. is supported by a Google Postdoctoral Fellowship.


{\small
\bibliographystyle{ieee_fullname}
\bibliography{egbib}
}

 \newpage
 
 \section{Appendix}
 \subsection*{A. Proof of Lemma 2}
 \begin{proof}
 Let $X_i$ be a random vector taking values in $\mathbb{R}^d$ with mean $\bm \mu=\mathbb{E}[X]$ and covariance matrix $\bm R=\mathbb{E}(X-\bm \mu)(X-\bm \mu)^T$. Given the $X_1, \dots, X_n$, the goal is to estimate $\bm \mu$. 
If $X$ has a multivariate Gaussian or sub-Gaussian distribution, the sample mean $  \bm{\bar \mu}_N=\frac{1}{N}\sum_{i=1}^{N}X_i $ is the result of MLE estimation, which satisfies, with probability at least $1-\delta$ 
\begin{equation}\label{dava21}
||\bm{\bar \mu}_N - \bm{ \mu}||\le \sqrt{\frac{\textrm{Tr}(\bm R)}{N}}+\sqrt{\frac{2\lambda_{\textrm{max}}\log(1/\delta)}{N}}
\end{equation}
where $\textrm{Tr}(\bm R)$ and $\lambda_{\textrm{max}}$ denote the trace and largest eigenvalue of the covariance
matrix $\bm R$, respectively \cite{hanson1971bound}. We already know  the truncated normal distribution mean and variance. Although, the truncated distribution is similar to Gaussian, we need to prove that it satisfies the sub-Gaussian distribution property so that we can use the bound in \eqref{dava21}.

The truncated distribution with mean $\bm \mu$ and covariance matrix $\bm R$ is a sub-Gaussian distribution.
A given distribution is sub-Gaussian if for all unit vectors $\{\bm v \in \mathbb{R}^d: ||\bm v||=1\}$
\cite{lugosi2019sub}, the following condition holds
\begin{equation}
    \mathbb{E}~ [\textrm{exp}(\lambda \langle \bm v, X-\bm \mu\rangle) ]\le \textrm{exp}(c\lambda^2\langle \bm v, \bm \Sigma \bm v \rangle).
\end{equation}
Assuming the hyperplane $\bm w^T X \ge 0$ truncated Normal distribution with mean zero and covariance matrix $\bm \Sigma$, the left hand side of the (13) can be computed as:
\begin{align}
    \mathbb{E}~ [\textrm{exp}(\lambda \langle \bm v, X-\bm \mu\rangle) ]=\int_{\mathcal{H}^+} \textrm {exp}(\lambda \bm v^T X) \phi_d({X}|\bm\Sigma) dX 
\end{align}
where $\mathcal{H}^+=\{X \in \mathbb{R}^d: \bm w^T X \ge 0 \}$. Since $\bm R$ is a symmetric, positive definite matrix, using Cholesky decomposition we can have $\bm R^{-1}=\bm \Psi^T \bm \Psi$ where $\bm \Psi$ is a non-singular, upper triangular matrix \cite{higham1990analysis}. By transforming the variables, we have $Y= \bm \Psi X$. Using $Y$, with some manipulation as in \cite{tallis1965plane}, one can get
\begin{equation}
        \mathbb{E}~ [\textrm{exp}(\lambda \langle \bm v, X-\bm \mu\rangle) ]= \textrm{exp} \left(\frac{1}{2}\lambda^2 \bm v^T \bm \Sigma \bm v \right)\Phi \left[\frac{\lambda \bm w^T \bm \Sigma \bm v}{\sigma}\right]
\end{equation}
and $\sigma^2= \bm w^T \bm \Sigma \bm w$, and $\Phi[.]$ is the cumulative
distribution function of the univariate normal distribution. Plugging $\bm \Sigma =\mathcal{I}$, one can get
\begin{align}
    \mathbb{E}~ [\textrm{exp}(\lambda \langle \bm v, X\rangle) ]& =  \textrm{exp} \left(\frac{1}{2}\lambda^2  \right)\Phi \left[{\lambda \bm w^T  \bm v}\right] \nonumber
    \\&\le\textrm{exp} \left(\frac{1}{2}\lambda^2  \right),
\end{align}
where the inequality is valid due to the fact that the CDF function is equal to 1 in the maximum. Comparing with the right hand side of the (13):
\begin{align}
    \textrm{exp} \left(\frac{1}{2}\lambda^2  \right) \le \textrm{exp} \left(\frac{1}{2}c\lambda^2  \right),
\end{align}
one can see that it is valid for any $c\ge 1$. Thus, the truncated Normal distribution is a sub-Gaussian distribution.
\end{proof}
The above proof is    consistent with our intuition as the truncated Gaussian has the tails approaching zero at least as fast as exponential distribution. The truncated part of the Gaussian is already equal to zero so there is no chance for being a heavy tailed distribution.
 Thus, the bound provided in \eqref{dava21} can be valid for our problem \cite{lugosi2019sub}.

Since the covariance matrix $\bm R$ is unknown, we need to find bounds for $\textrm{Tr}(\bm R)$ and $\lambda_{\textrm{max}}$ as well. It can easily be obtained that 
\begin{equation}
    \textrm{Tr}(\bm R)= d+c_2 \bm w^T \bm w = d+c_2
\end{equation}
In order to obtain the maximum eigenvalue of the $\bm R$, we use Weyl's inequality to have an upper bound for largest eigenvalue of the covariance matrix as \cite{marsli2015bounds}:
\begin{equation}
    \lambda_{\textrm{max}}(\bm A + \bm B) \le  \lambda_{\textrm{max}}(\bm A ) +  \lambda_{\textrm{max}}(\bm B)
\end{equation}
The largest eigenvalue for the identity matrix $\mathcal{I}$ is $1$. For the rank-1 matrix $c_2 \bm w \bm w^T$ which is the outer product of the normal vector is given by:
\begin{equation}
    \lambda_{\textrm{max}}(c_2 \bm w \bm w^T)=c_2\textrm{Tr}(\bm w \bm w^T)=c_2 \bm w^T \bm w= c_2
\end{equation}
which immediately results in $ \lambda_{\textrm{max}}(\bm R) \le 1+ c_2$.
\textcolor{black}{Substituting the above values to the (12), the sample mean $  \bm{\bar \mu}_N=\frac{1}{N}\sum_{i=1}^{N}X_i $ is the result of MLE estimation, which satisfies, with probability at least $1-\delta$ 
\begin{equation}\label{eq21}
\|\bm{\bar \mu}_N - \bm{ \mu}\|\le \sqrt{\frac{d+c_2}{N}}+\sqrt{\frac{2(1+c_2)\log(1/\delta)}{N}}
\end{equation}
This bound can provide an upper bound with probability at least $1-\delta$ for the error of the sample mean while getting $N$ queries from the neural network.
}
\begin{equation}\label{eq31}
    \frac{R}{\sin({\theta_t})}=\frac{r_{t}}{\sin({\theta_{t+1}})}
\end{equation}
  \begin{figure}[!t]
	\centering
	\hspace*{-0.0in}
	\includegraphics[width=0.46\textwidth]{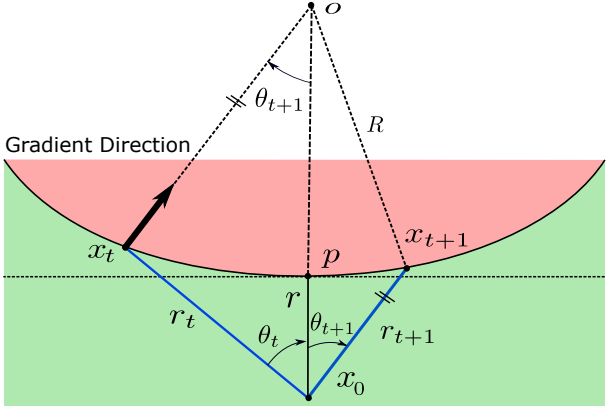}
	\vspace{-0.0in}
	\caption{Convex decision boundary with  bounded curvature.}
	\label{fig11}
\end{figure}
 \subsection*{B. Proof of Theorem 1}
  In the following subsections, we consider two cases for the curvature of the boundary.
 \subsubsection*{Convex Curved Bounded Boundary} We assume that the curvature of the boundary is convex as given in Fig.~\ref{fig11}. As given in \cite{fawzi2016robustness}, if $\theta_t$ satisfies the two assumptions $\tan^2{(\theta_t)} \le 0.2R /r$ and $r/R <1$, the value for $\|\bm x_t -\bm x_0\|_2=r_t$ is given as follows:
 \begin{equation}
     r_t=-(R-r)\cos(\theta_t)+\sqrt{(R-r)^2 \cos^2(\theta_t)+2Rr-r^2}
 \end{equation}
 where $\|\bm x_{t+1} -\bm x_0\|_2=r_{t+1}$ can be obtained in a similar way.
 It can be observed that the value of the $r_t$ is an increasing function of the $\theta_t$ because:
\begin{align}
\frac{\partial r_t}{\partial \theta_t}& =  (R-r)\sin(\theta_t) \nonumber
    \\&-\frac{(R-r)^2\cos(\theta_t)\sin(\theta_t)}{\sqrt{(R-r)^2\cos^2(\theta_t)+2Rr-r^2}},
\end{align}
 Setting $\frac{\partial r_t}{\partial \theta_t} >0$, with some manipulations one can get $2R >r$ which shows that $r_t$ is an increasing function of the $\theta_t$. Thus, if we can show that $\theta_t > \theta_{t+1}$, it means that $r_t>r_{t+1}$ which means that $r_t $ can converge to $r$.
 Here, we assume that the given image is in the vicinity of the boundary $r/R<1$. The line connecting point $ o$ to $ x_0$ intersects the two parallel lines. Based on the law of sines, one can get
Since $r_t < R$, one can  conclude that $\theta_t > \theta_{t+1}$ using the sines law.  Thus, as $r_t$ is an increasing function of $\theta_t$, we can get $r_{t+1} < r_{t}$.
Thus, after several iterations, the following update rule  \begin{equation}\label{eq28}
     \bm x_{t}=\bm x_0 + r_{t} \hat {\bm w}_{N_t}
 \end{equation}
 converges to the minimum perturbation $r$.
 \begin{figure}[!t]
	\centering
	\hspace*{-0.0in}
	\includegraphics[width=0.48\textwidth]{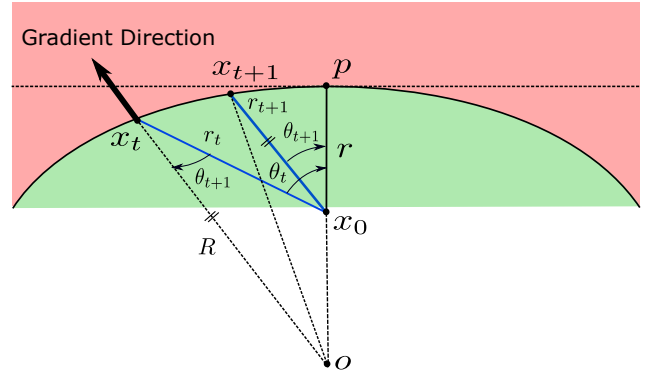}
	\label{fig:antenna_LMS}\vspace{-0.0in}
	\caption{Concave decision boundary with  bounded curvature.}
	\label{fig1}
\end{figure}

Applying the sine law for $k$ iterations, one can get the following equation using \eqref{eq31}:
\begin{equation}
    \sin(\theta_{t})=\frac{\prod_{i=0}^{t}r_{i}}{R^t} \sin(\theta_0)
\end{equation}
We know that $r_t < R$ and in each iteration, it gets smaller and smaller. Thus, for the convergence, we consider the worst case. We know that $\max_{i = 0, 1, ..., t } \{r_{i}\} = r_t$. Thus,
To bound this, we can have:
\begin{equation}
    \sin(\theta_{t+K})=\frac{\prod_{k=0}^{K}r_{t+k}}{R^K} \sin(\theta_t) \le (\frac{r_t}{R})^K\sin(\theta_t)
\end{equation}
where can be reduced to
\begin{equation}
    \sin(\theta_{t+K}) \le (\frac{r_t}{R})^K\sin(\theta_t)
\end{equation}
This shows that $\sin(\theta_{t+K})$ converges to zero exponentially since $r_t < R$. Thus, $\theta_{t+k}$ goes to zero which results that the in coinciding the $r_t$ and $r$ in the same magnitude. Thus, we have
\begin{equation}
    \lim_{k\to\infty} r_{t+k} = r
\end{equation}
We already know that
 \begin{align}
r_{t+1} = & -(R-r)\cos(\theta_{t+1}) \nonumber
    \\&+\sqrt{(R-r)^2 \cos^2(\theta_{t+1})+2Rr-r^2},
\end{align}
 Considering the cosine law, based on the figure, we can see that
  \begin{equation}
     r_{t}^2=(R+r)^2 +R^2-2R(R+r)\cos(\theta_{t+1})
 \end{equation}
 By combining the above equations and eliminating the $\cos(\theta_{t+1})$, one can get:
\begin{align}\label{eq39}
r_{t+1}& =  -(R-r)\frac{(R+r)^2+R^2-r_t^2}{2R(R+r)}\nonumber
    \\&+\sqrt{ \frac{(R-r)^2((R+r)^2+R^2-r_t^2)^2}{4R^2(R+r)^2}+2Rr-r^2},
\end{align}
Plugging \eqref{eq39} into the following limit,
\begin{equation}\label{lim}
    \lim_{t\to\infty} \frac{r_{t+1}-r}{r_t-r}
\end{equation}
for $t \to \infty$, we get $\frac{0}{0}$. Thus, using the L'Hospital's Rule, we take the derivative of the numerator and the denominator as:
\begin{align}
\frac{\partial r_{t+1}}{\partial r_t}& = - \frac{r_t(R-r)}{R(R+r)}& \nonumber
    \\+&\frac{((R+r)^2+R^2-r_t^2)r_t}{2\sqrt{ \frac{(R-r)^2((R+r)^2+R^2-r_t^2)^2}{4R^2(R+r)^2}+2Rr-r^2}}\frac{(R-r)^2}{R^2(R+r)^2}&
\end{align}
Having $t \to \infty$, we can get $r_t \to r$,  since we have $\hat r_t \to r_t$, thus:
\begin{align}
\lim_{t\to\infty} \frac{\hat r_{t+1}-r}{\hat r_t-r}  =\frac{r^2(R-r)}{R^2(R+r)} = \lambda <1
\end{align}
 As $r<R$,  the rate of convergence $\lambda \in (0,1)$ which completes the proof.

 \subsubsection*{Concave Curved Bounded Boundary}
 As in  \cite{fawzi2016robustness}, the  value for $\|\bm x_t -\bm x_0\|_2=r_t$ is given as follows:
 \begin{equation}
     r_t=(R+r)\cos(\theta_t)-\sqrt{(R+r)^2 \cos^2(\theta_t)-2Rr-r^2}
 \end{equation}
 where $\|\bm x_{t+1} -\bm x_0\|_2=r_{t+1}$ can be obtained in a similar way. It can easily be seen that the $\theta_t > \theta_{t+1}$. Assuming $r/R<1$, $r_t$ is a decreasing function with respect to $\theta_t$ which results in $r_t<r_{t+1}$. Similar proof of convergence can be obtained for this case as well. 
 
  \subsection*{C. Proof of Theorem 2}

Given the point $r_{t-1}$, the goal is to find the estimate of the $\hat r_{t}$ with limited query. Assuming the normalized version of the true gradient $ {\bm w_t}=  {\bm \mu_t} / \|  {\bm \mu_t} \|_2$, we have  
 \begin{equation}\label{davaaa8a}
\|\bm{\hat w}_{N_{t}} - \bm{ w_t}\|\le {\frac{\gamma}{\sqrt{N_{t}}}}
\end{equation}
where $\gamma = \sqrt{\textrm{Tr}{(\bm R})} + \sqrt{2 \lambda_{\max} \log(1/\delta)}$, $\bm{\hat w}_{N_t}$ is the estimated gradient at iteration $t$ and $N_t$ is the number of queries to estimate the gradient at point $\bm x_{t-1}$. 
Based on the reverse triangle inequality  $ \|\bm x\|-\|\bm y\| \le  \|\bm x - \bm y\|$, we can have
 \begin{equation}\label{tetassdist}
 \|\bm{\hat w}_{N_t} \| - 1  \le \|\bm{\hat w}_{N_t} - \bm{ w_t}\|  \le {\frac{\gamma}{\sqrt{N_t}}}.
\end{equation}
Multiplying by $r_t$, we have:
 \begin{equation}\label{tetadssist}
  r_t\ - {\frac{\gamma  r_t}{\sqrt{N_t}}} \le  \hat r_t    \le  r_t + {\frac{\gamma  r_t}{\sqrt{N_t}}}.
\end{equation}
where $\hat r_t= r_t  \|\bm{\hat w}_{N_t} \|$. 
  Here, we conduct the analysis in the limit sense and we observe in the simulations that it is valid in limited iterations as well. Given $r_{t-1}$, for large $t$, we have:
\begin{equation}\label{ddss}
 r_{t}-r \approx \lambda (r_{t-1}-r) 
\end{equation}
Considering the best and worst case for the estimated gradient, we can find the following bound. In particular, the best case is the case in which all the gradient errors are constructive and make the $\hat r_t$ in each iteration smaller than $r_t$. In contrast, the worst case happens when all the gradients directions are destructive and make the $\hat r_t$ greater than $r_t$. In practice, however,  what is happening is something in between. Substituting $r_t$ from \eqref{tetadssist} in \eqref{ddss}, one can obtain:
\begin{equation}
  \lambda (r_{t-1}-r) -  \frac{\gamma  r_{t}}{\sqrt{N_{t}}} \le \hat r_{t}-r \le \lambda (r_{t-1}-r) +  \frac{\gamma  r_{t}}{\sqrt{N_{t}}}
\end{equation}
 By using the iterative equation, one can get the following bound:
\begin{align}
    \lambda^{t} (r_{0}-r) - e(\bm N) \le  \hat r_{t}-r \le \lambda^{t} (r_{0}-r) +  e(\bm N)
\end{align}
where $   e(\bm N)= \gamma\sum_{i=1}^{t}\frac{\lambda^{t-i} r_{i}}{\sqrt{N_{i}}}$ is the error due to limited number of queries.

\subsection*{D. Proof of Theorem 3}
It can easily be observed that the optimization problem is convex. Thus, the duality gap between this problem and
its dual optimization problem is zero. Therefore, we can solve
the given problem by solving its dual problem. The Lagrangian is given by:
\begin{equation}
    \mathcal{L}(\bm N ,\alpha)= \sum_{i=1}^{T}  \frac{\lambda^{-i}r_i}{\sqrt{N_i}} + \alpha \left(\sum_{i=1}^{T} N_i -N \right )
\end{equation}
where $\alpha$ is the  non-negative dual variable associated
with the budget constraint. The KKT conditions  are given as follows \cite{boyd2004convex}:
\begin{align}\label{kkt1}
 &&& \frac{\partial   \mathcal{L}(\bm N ,\alpha)}{\partial N_t}=0, ~ \forall i\\ \label{kkt2}
& && \alpha \left(\sum_{i=0}^{T}N_i -N \right)=0  \\ \label{kkt3}
&&& \sum_{i=1}^{T} N_i \le N  
\end{align}
Based on \eqref{kkt1}, taking the derivative and setting equal to zero, we can have 
\begin{equation}\label{query}
   N_t =\left(\frac{\lambda^{-t}r_t}{2\alpha}\right)^{\frac{2}{3}} 
\end{equation}
 We see that the constraint holds with equality. Assume that $\sum_{i=0}^{t} N_i \ne N  $, then based on \eqref{kkt2}, $\alpha = 0$. If $\alpha=0$ then based on \eqref{query}, we have $N_i= \infty, ~\forall i$ which contradicts with \eqref{kkt3}.
 Substituting \eqref{query} in $\sum_{i=0}^{t} N_i = N  $, the Lagrangian multiplier can be obtained as
 \begin{equation}
     \alpha^{\frac{2}{3}}= \frac{1}{2^{\frac{2}{3}}}\frac{\sum_{i=1}^T(\lambda^{-i}r_i )^{\frac{2}{3}}}{N}
 \end{equation}
 Substituting $\alpha$ in \eqref{query}, one can get the optimal number of queries as:
 \begin{equation}
     N_t^*= \frac{( \lambda^{-t} r_t)^{\frac{2}{3}}}{\sum_{i=1}^T(\lambda^{-i} r_i)^{\frac{2}{3}}}
     N
 \end{equation}
 For $t \to \infty$, we have $r_t \to r $. Based on this, the ratio of the optimal number if queries for each iteration is given by:
  \begin{equation}
     N_t^* \approx \frac{ \lambda^{-\frac{2}{3}t}}{\sum_{i=1}^T\lambda^{-\frac{2}{3}i}}
     N
 \end{equation}
 This equation shows that the distribution of the queries should be increased by a factor of $\lambda^{-\frac{2}{3}}$ where $0< \lambda < 1$. By approximation, we have
  \begin{equation}
     \frac{N_{t+1}^*}{N_{t}^*} \approx \lambda^{-\frac{2}{3}}
 \end{equation}
 which completes the proof. 

  \begin{figure*}[ht]
\centering
  \includegraphics[width=\textwidth,height=20.3cm]{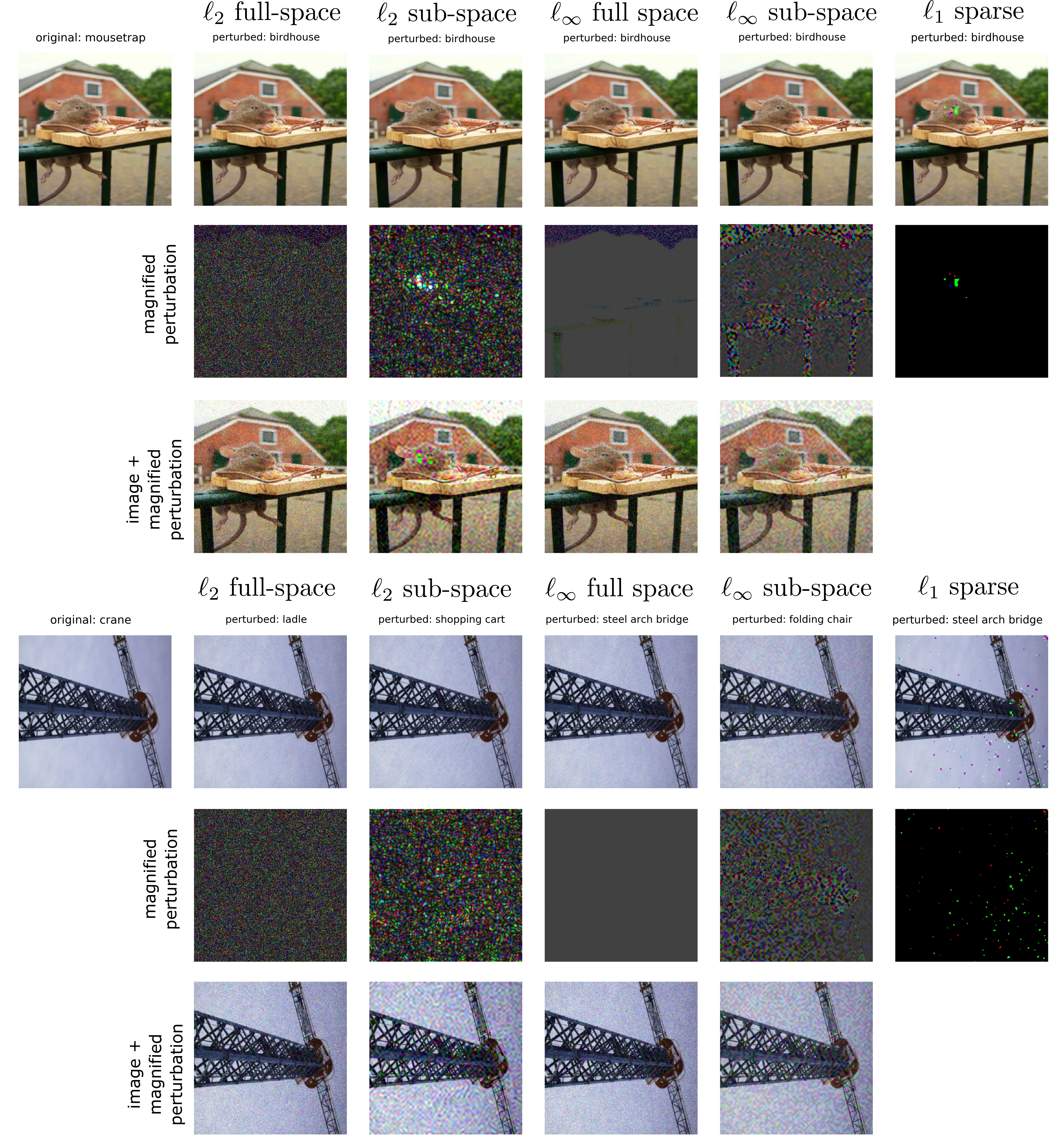}
  \caption{Original images and adversarial perturbations generated by GeoDA for $\ell_2$ fullspace, $\ell_2$ subspace, $\ell_\infty$ fullspace, $\ell_\infty$ subspace, and $\ell_1$ sparse with $N=10000$ queries.}
  \label{fig111}
\end{figure*}

  \begin{figure*}[ht]
\centering
  \includegraphics[width=\textwidth,height=20.3cm]{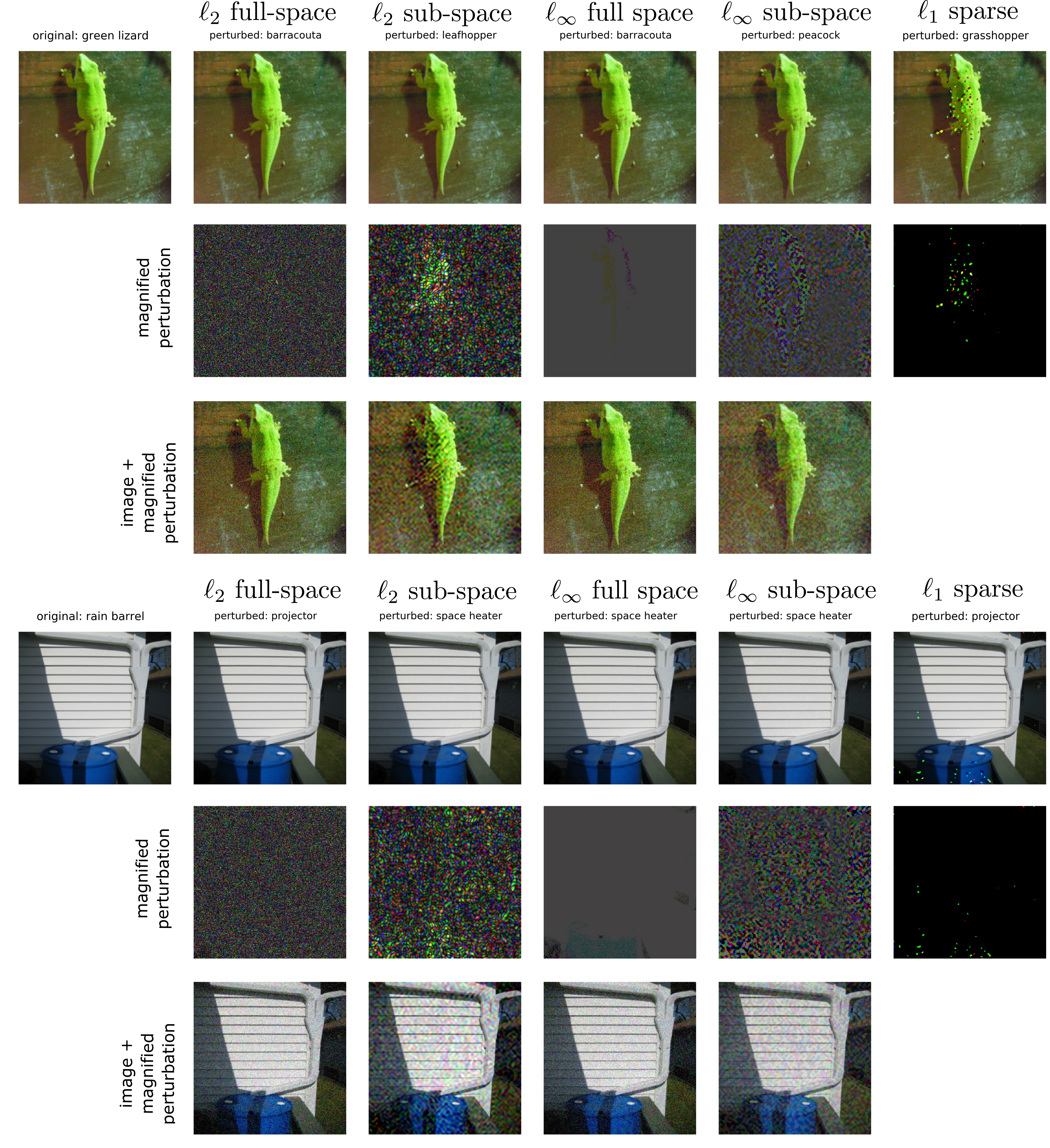}
  \caption{Original images and adversarial perturbations generated by GeoDA for $\ell_2$ fullspace, $\ell_2$ subspace, $\ell_\infty$ fullspace, $\ell_\infty$ subspace, and $\ell_1$ sparse with $N=10000$ queries.}
    \label{fig222}
\end{figure*}
  \begin{figure*}[ht]
\centering
  \includegraphics[width=\textwidth,height=20.3cm]{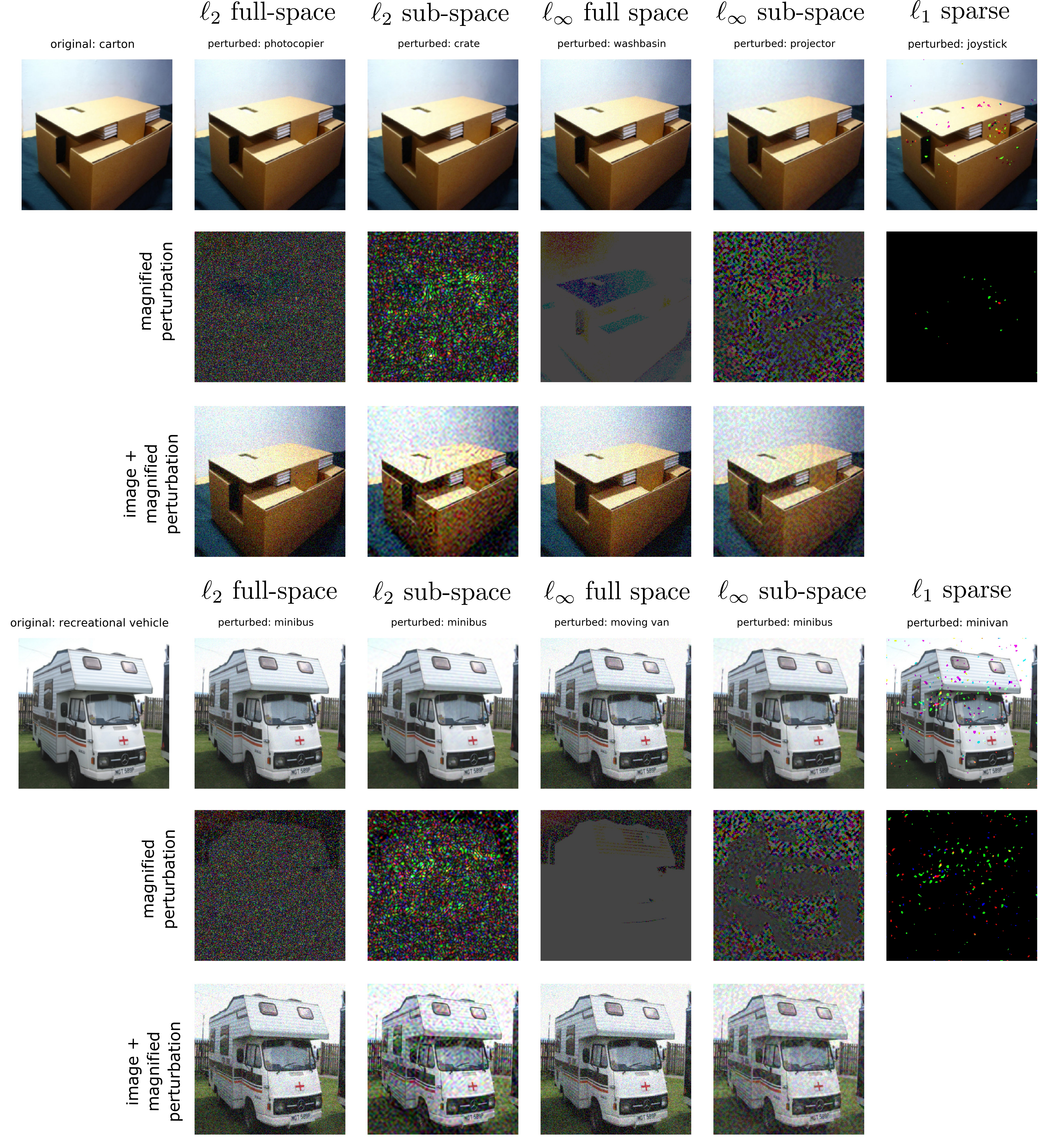}
  \caption{Original images and adversarial perturbations generated by GeoDA for $\ell_2$ fullspace, $\ell_2$ subspace, $\ell_\infty$ fullspace, $\ell_\infty$ subspace, and $\ell_1$ sparse with $N=10000$ queries.}
      \label{fig333}
\end{figure*}
  \begin{figure*}[ht]
\centering
  \includegraphics[width=\textwidth,height=20.3cm]{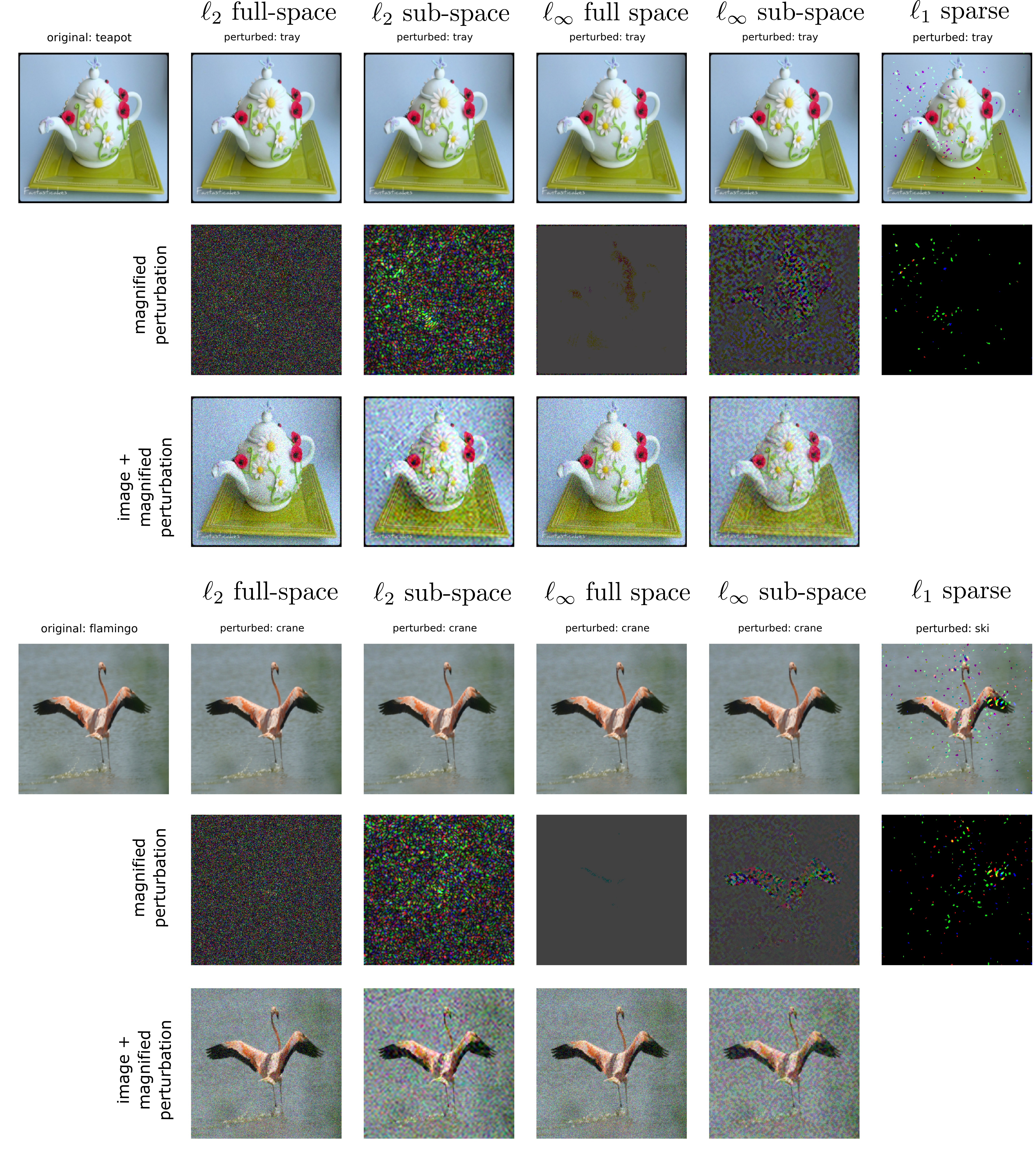}
  \caption{Original images and adversarial perturbations generated by GeoDA for $\ell_2$ fullspace, $\ell_2$ subspace, $\ell_\infty$ fullspace, $\ell_\infty$ subspace, and $\ell_1$ sparse with $N=10000$ queries.}
      \label{fig444}
\end{figure*}

 \section{Additional experiment results}
 Here we show more experiments on the performance of GeoDA on different $\ell_p$ norms. In Figs.~\ref{fig111}, Figs.~\ref{fig222}, Figs.~\ref{fig333}, and Figs.~\ref{fig444}, we have generated adversarial examples using GeoDA. For each image, the first row consists of (from left to right) original image, $\ell_2$ fullspace  adversarial example, $\ell_2$ subspace  adversarial example, $\ell_\infty$ fullspace  adversarial example, $\ell_\infty$ subspace  adversarial example, and $\ell_1$ adversarial example, respectively. However, as can be seen the perturbations are not quite visible in the actual adversarial examples in the first row. In the second row, we show the magnified version of perturbations for $\ell_2$ and $\ell_\infty$. To do so, the norm of all the perturbations is magnified to 100 given that the images coordinate normalized to the 0 to 1 scale. For the sparse case, we do not magnify the perturbations as they are visible and equal to their maximum (minimum) values. Finally, in the third row, we added a magnified version of the perturbation with norm of 30 to have a better visualization. 
 \begin{table} [h]
\centering 
\begin{tabular}{c | c |c c c c  } 
\toprule 

 & \small{\textbf{Queries}} & \small{\textbf{ ResNet-50}} & \small{\textbf{ ResNet-101} }    \\

\midrule  \midrule

  & \small{500 } & \small{11.76}  & \small{17.91}  \\
{{ GeoDA ($\ell_2$)}} & \small{2000 }  &   \small{3.35} & \small{6.38}   \\
  & \small{10000 } & \small{1.06}  & \small{1.87}  \\

\bottomrule
\end{tabular}

\caption{The performance comparison of  GeoDA on different ResNet image classifiers.} 
\label{netsss} 
\end{table}

 In Table \ref{netsss}, we have compared the performance of GeoDA with different deep network image classifiers.  The proposed algorithm GeoDA follows almost the same trend on a wide variety of deep networks.  The reason is that the core assumption of GeoDA, i.e. boundary has a low mean curvature in vicinity of the datapoints,  is verified empirically for a wide variety of deep networks. We can provide the experimental results on different networks.

\end{document}